\documentclass[Afour,sageh,times]{sagej}

\usepackage{moreverb,url}

\usepackage[colorlinks,bookmarksopen,bookmarksnumbered,citecolor=red,urlcolor=red]{hyperref}

\newcommand\BibTeX{{\rmfamily B\kern-.05em \textsc{i\kern-.025em b}\kern-.08em
T\kern-.1667em\lower.7ex\hbox{E}\kern-.125emX}}

\def\volumeyear{2016}

\usepackage{amsmath}
\usepackage{amsfonts}
\usepackage{amssymb}
\usepackage{mathtools}
\usepackage{arydshln}
\usepackage{booktabs,caption}
\usepackage[flushleft]{threeparttable}
\usepackage{graphicx}
\usepackage{subcaption}
\usepackage{caption}
\usepackage{multirow}
\usepackage[algo2e,algoruled,boxed,lined,norelsize]{algorithm2e} 
\graphicspath{{images/}}
\usepackage[toc,page]{appendix}
\usepackage{enumitem}
\usepackage{amsthm}

\allowdisplaybreaks
\DeclarePairedDelimiter{\ceil}{\lceil}{\rceil}
\newtheorem{definition}{Definition}
\newtheorem{lemma}{Lemma}
\newtheorem{theorem}{Theorem}

\DeclareMathOperator{\rank}{rank}
\newtheorem{corollary}{Corollary}
\newtheorem*{remark}{Remark}

\DeclarePairedDelimiter{\norm}{\lVert}{\rVert}

\SetKwInOut{Initial}{Initial}
\SetKwInOut{Parameter}{Parameters}
\newtheorem{assumption}{Assumption}

\newcommand{\col}{{\rm col\;}}

\def\qed{ \rule{.1in}{.1in}}
\definecolor{mygreen}{rgb}{0,0.6, 0}

\begin{document}

\runninghead{Jin et al.}

\title{Inverse Optimal Control from Incomplete Trajectory Observations}

\author{Wanxin Jin\affilnum{1}, Dana Kuli\'c\affilnum{2}, Shaoshuai Mou\affilnum{1}and Sandra Hirche\affilnum{3}}

\affiliation{
\affilnum{1}School of Aeronautics and Astronautics, Purdue University, 47906 IN, USA. Email: wanxinjin@gmail.com, mous@purdue.edu\\
\affilnum{2}Monash University, Melbourne VIC 3800, Australia.\\
\affilnum{3}Chair of Information-oriented Control, Technical University of Munich,  D-80290 Munich, Germany.
}

\corrauth{Wanxin Jin, School of Aeronautics and Astronautics, Purdue University, 47906 Indiana, USA}

\email{wanxinjin@gmail.com}

\begin{abstract}
This article develops a methodology that enables learning an objective function of an optimal control system from  incomplete trajectory observations. The objective function  is assumed  to be a weighted sum of  features (or basis functions) with unknown weights, and the observed data is a segment of a  trajectory of  system states and inputs. 
The  proposed technique introduces the concept of the recovery matrix to establish the relationship between any available  segment of the  trajectory and the  weights of given candidate features. The rank of the recovery matrix indicates whether a subset of relevant features can be found among the candidate features and the corresponding weights can be learned from the segment data. The recovery matrix can be obtained iteratively and its rank  non-decreasing property shows that additional observations may contribute to the objective learning. Based on the  recovery matrix, a method for using incomplete trajectory observations to learn the weights of selected features is established, and an incremental inverse optimal control algorithm is developed  by automatically finding the minimal required observation. The effectiveness of the proposed method is demonstrated on a linear quadratic regulator system and a simulated robot manipulator.

\end{abstract}


\maketitle

\section{Introduction}
Inverse optimal control (IOC), also known as inverse reinforcement learning, solves a problem of {finding an objective   (e.g., cost or reward) function to explain behavioral observations of an optimal control system}
\citep{kalman1964linear,ng2000algorithms}. Successful applications {of IOC techniques} are broad,  including imitation learning \citep{abbeel2010autonomous,kolter2008hierarchical}, where a learner mimics  an expert by inferring an objective function from the expert's demonstrations, autonomous driving \citep{kuderer2015learning}, where  human driving preference is learned and transferred to a vehicle controller, human-robot shared systems \citep{mainprice2015predicting,mainprice2016goal}, where the  intentionality of a human partner is estimated to enable  motion prediction and smooth coordination, and  human motion analysis \citep{jin2019inverse,lin2016human}, where  principles of human motor control are investigated.

The most common strategy used in IOC is to parametrize an unknown objective function as a weighted sum of  relevant features (or basis functions) with unknown weights \citep{ng2000algorithms,abbeel2004apprenticeship,ratliff2006maximum,ziebart2008maximum,mombaur2010human,englert2017inverse}. Different approaches have been developed to estimate the weights of the features  given a  \emph{full observation} of the system's optimal trajectory over a complete motion horizon. These methods cannot deal with the case when only \emph{incomplete trajectory observations} are available, that is, only a portion/segment  of the system's  trajectory within a small time interval of  the  horizon  is available. A method capable of learning objective functions from  incomplete trajectory data  will be beneficial for multiple reasons:  first, a full observation  of the system's trajectory over a complete   time horizon may not be accessible due to limited sensing capabilities, sensor failures, occlusions, etc; second, the computational cost of existing IOC techniques based on full trajectory data may be large; and third,  learning objective functions from incomplete trajectory observations may help to address some  other challenging problems such as  learning time-varying objective functions \citep{jin2019inverse}, online motion prediction \citep{perez2015fast}, and learning from human corrections \citep{bajcsy2018learning}. Under these motivations, this article aims to develop a technique to enable  learning an objective function   only from  incomplete trajectory observations.

\subsection{Related Work}
Existing IOC methods can be  categorized based on whether the forward optimal control problem needs to be computed within the learning process. The first category of existing  works is based on a nested architecture, where the feature weights are updated in an outer loop while the corresponding optimal control system is solved in an inner loop.  Different methods of this type focus on  different strategies to update the feature weights in the outer layer. Representative methods include \citep{abbeel2004apprenticeship}, {where the feature weights are updated towards  matching the feature values of the reproduced optimal trajectories with the demonstrations,} \citep{ratliff2006maximum}, where the feature weights are solved by maximizing the margin between the objective function value of the observed trajectories and the value of any simulated optimal trajectories, and \citep{ziebart2008maximum}, where the feature  weights are optimized such that the probability distribution of  system's trajectories maximizes the entropy while matching the empirical feature values of   demonstrations. These nested IOC  methods have been successfully applied to   humanoid locomotion \citep{park2013inverse}, autonomous vehicles \citep{kuderer2015learning},  robot navigation \citep{vasquez2014inverse}, learning from human corrections \citep{bajcsy2017learning,jin2020learning}, etc. In \citep{mombaur2010human}, the weights are learned by minimizing the deviation of the reproduced optimal trajectory from the observed one; the similar methods are applied to studying human walking \citep{clever2016inverse} and arm motion  \citep{berret2011evidence}.

{A drawback of} the nested IOC methods is the need to solve optimal control problems repeatedly in the inner loop, thus those methods usually suffer from huge computational cost. This motivates the second line of IOC methods, which seek to directly solve for the unknown feature weights. A key idea used in these methods is to establish  optimality conditions  which the observed optimal data  must satisfy.  {For example,} in \citep{keshavarz2011imputing},  the Karush-Kuhn-Tucker (KKT) \citep{boyd2004convex} optimality conditions are established, based on which the feature weights  are then solved by minimizing a loss that  quantifies the violation of such conditions by the observed data.  In \citep{puydupin2012convex}, the authors apply such the KKT-based method to solve IOC problems and study the objective function underlying human locomotion. In \citep{johnson2013inverse}, the Pontryagin's Minimum Principle \citep{pontryagin1962mathematical} is utilized to formulate a residual optimization over the unknown weights. These  methods have been successfully applied to the locomotion analysis \citep{aghasadeghi2014inverse,puydupin2012convex}, walking path generation \citep{papadopoulos2016generation},  human motion segmentation \citep{lin2016human,jin2019inverse}, etc. In   \citep{englert2017inverse}, the authors propose  an inverse KKT method  to enable a robot to learn manipulation tasks. Recently, along this direction, the recoverability for IOC problems has been investigated.  For example, when an optimal control system remains at an equilibrium point, although its trajectory still satisfies the optimality conditions, it is uninformative for  learning the objective function. This issue is discussed in \citep{molloy2016discrete,molloy2017finite}, where a sufficient condition for recovering  weights  from full trajectory observations is proposed.

Existing IOC techniques discussed above require a full observation of a system  trajectory,  that is,  optimal trajectory data  of the system states and inputs over a complete motion horizon.  To our best knowledge,  the IOC problems based on  incomplete trajectory observations are rarely investigated. By an incomplete trajectory observation, we mean that the observed data is a segment or portion of  the system trajectory within a small time interval of the  horizon. We  consider the objective learning using incomplete trajectory data mainly due to the following  motivations. First,   {in certain practical cases}, for example,  due to limited sensing, sensors' failure, or occlusions, a full observation of  system's trajectory data may not be available. Second,  although direct IOC methods improve the computation efficiency compared to the nested counterparts, the computational cost is still significant especially   when handling complex systems with high-dimensional action/state space and long time horizons. Third,  successfully learning objective functions  only using incomplete trajectory data would potentially benefit for addressing many challenging problems such as identifying time-varying objective functions \citep{jin2019inverse}, learning from human  corrections \citep{bajcsy2018learning}, online long-term motion prediction \citep{mainprice2016goal}, etc.

\subsection{Contributions}
This article develops a methodology to learn an objective function of  an optimal control system using  incomplete trajectory observations. The proposed key concept to achieve this goal is the  \emph{recovery matrix}, which is defined on \emph{any  segment data} of the system trajectory and a given candidate feature set. We show that learning of the objective function is related to the rank and kernel properties of the recovery matrix. Different from existing methods, the recovery matrix also captures the \emph{unseen} future information in addition to the available data by an unknown costate variable, which is   jointly estimated  along with the unknown feature weights. By investigating the properties of the recovery matrix, the following insights to solving IOC problems are enabled:
\begin{itemize}
		\item[1)] The rank of  the recovery matrix indicates whether  an observation of incomplete trajectory  data is sufficient for learning the feature weights;
		\item[2)] 
		Additional observation data can contribute to learning the unknown objective function, or at least not degrade the  learning; and irrelevant features can be identified.
		\item[3)] The IOC can be solved by incrementally  incorporating the observation of each data point along the trajectory.
\end{itemize}
\noindent
Based on the  recovery matrix, an IOC approach based on incomplete trajectory observations  is established, and an incremental IOC algorithm is developed  by automatically finding the minimal required
observation.

\smallskip
 
The structure of this article  is as follows. Section 2 states the problem. Section 3 develops   the recovery matrix and its properties. Section 4 presents the IOC method and algorithm  using incomplete trajectory observations. Section 5 conducts numerical  experiments, and Section 6 draws  conclusions.

\subsection*{Notation}
The column operator $\col \{\boldsymbol{x}_1,...,\boldsymbol{x}_k\}$ stacks its arguments into a column.
$\boldsymbol{x}_{k_1:k_2}$ means a column stack of  $\boldsymbol{x}$ indexed from $k_1$ to $k_2$ ($k_1\leq k_2$), that is,  $\boldsymbol{x}_{k_1:k_2}=\col \{\boldsymbol{x}_{k_1},...\boldsymbol{x}_{k_2}\}$.
$\boldsymbol{A}$ (bold) denotes a block matrix.
Given a vector function $\boldsymbol{f}(\boldsymbol{x})$ and a constant $\boldsymbol{x}^*$, $\frac{\partial \boldsymbol{f}}{\partial \boldsymbol{x}^*}$ denotes the Jacobian matrix with respect to $\boldsymbol{x}$ evaluated at $\boldsymbol{x}^*$.
The zero matrix and vector are written as $\boldsymbol{0}$, and the identity matrix as ${\boldsymbol{I}}$, both with appropriate dimensions. $\boldsymbol{A}^\prime$ is the transpose of matrix $\boldsymbol{A}$. $\sigma_{i}(\boldsymbol{A})$ denotes the $i$th smallest singular value of matrix $\boldsymbol{A}$, e.g., $\sigma_{1}(\boldsymbol{A})$ is the smallest singular value.
$\ker \boldsymbol{A}$ denotes the kernel of matrix $\boldsymbol{A}$.

\section{Problem Formulation}\label{problemformulation}
Consider an optimal control system with the following discrete-time dynamics and initial condition:
 {\begin{equation} \label{dynamics}
	\boldsymbol{x}_{k+1}=\boldsymbol{\boldsymbol{f}}(\boldsymbol{x}_k,\boldsymbol{u}_{k}), \quad \boldsymbol{x}_0\in\mathbb{R}^n, \\
	\end{equation}}%
where the vector function $\boldsymbol{f}: \mathbb{R}^n \times\mathbb{R}^m \mapsto \mathbb{R}^n $ is  differentiable; $\boldsymbol{x}_k \in \mathbb{R}^n$ is the system state; $\boldsymbol{u}_k \in \mathbb{R}^m$ is the control input; and $k=0,1,\cdots$ is the time step. Suppose  a trajectory of  states and inputs over a horizon $T$,
 {\begin{equation}\label{traj}
\boldsymbol{\xi}=\{\boldsymbol{\xi}_k:  k=0,1,...,T\}\, \,\text{ with }\,\, \,\boldsymbol{\xi}_k{=}(\boldsymbol{x}^*_{k}, \,\boldsymbol{u}^*_{k}),
\end{equation}}
(locally) minimizes a cost function 
\begin{equation} \label{costfunction}
J(\boldsymbol{\xi})= \sum\nolimits_{ {k=0}}^{T} {{\boldsymbol{\omega}}^{\prime}\boldsymbol{\phi}}^*(\boldsymbol{x}_k,\boldsymbol{u}_k),
\end{equation}
where ${\boldsymbol{\omega}}^{\prime}\boldsymbol{\phi}^*(\cdot,\cdot)$ is the \emph{running cost}. Here  $\boldsymbol{\phi}:\mathbb{R}^n \times\mathbb{R}^m \mapsto \mathbb{R}^s$ is called a \emph{relevant feature vector} and  defined as a column of a \emph{relevant feature set}
\begin{equation} \label{relevantfeatureset}
\mathcal{F}^*=\{\phi_1^*,\phi_2^*,\cdots,\phi_s^*\},
\end{equation}
that is, $\boldsymbol{\phi}^*=\col{\mathcal{F}}^*$, with $\phi_i^*$ being the $i$th  \emph{feature}  for the running cost, and $\boldsymbol{\omega}\in \mathbb{R}^s$ is called the \emph{weight vector}, with the $i$th entry $\omega_i$  corresponding to  $\phi_i^*$. This type of weighted-feature objective function is commonly used in  objective learning problems  \citep{abbeel2004apprenticeship,molloy2017finite,ziebart2008maximum}, and has been successfully applied in a wide range of real-world applications \citep{englert2017inverse,kuderer2015learning,lin2016human,bajcsy2018learning}.  {The dynamics (\ref{dynamics}) and cost function  (\ref{costfunction}) can represent  different optimal control settings as follows. (I) Finite-horizon free-end optimal control: the finite horizon $T$ is given  but the final state $\boldsymbol{x}_{T+1}$ is free, i.e., no  constraint on $\boldsymbol{x}_{T+1}$; (II) finite-horizon fixed-end optimal control: both the finite horizon $T$  and  the final state $\boldsymbol{x}_{T+1}=\boldsymbol{x}_{\text{goal}}$ are given;   and (III) infinite-horizon optimal control: $T=\infty$. Besides, one can consider the  finite-horizon optimal control, where the   final state $\boldsymbol{x}_{T+1}$ is penalized using a final cost term added to (\ref{costfunction}), and this case can be viewed as an extension similar to (II). In our following expositions, we only focus on the first three settings.}

In inverse optimal control (IOC) problems, one is given a relevant feature set $\mathcal{F}^*$, the goal is to obtain an estimate of the weights $\boldsymbol{\omega}$ corresponding to these features from observations of $\boldsymbol{\xi}$. Note that  $\boldsymbol\omega$ can only be determined up to a non-zero scaling factor \citep{molloy2017finite,keshavarz2011imputing}, because any $c\boldsymbol\omega$ with $c> 0$ will  lead to the same  trajectory  $\boldsymbol{\xi}$. Hence we say an estimate  $\hat{\boldsymbol\omega}$ is a \emph{successful estimate} of $\boldsymbol\omega$ if $\hat{\boldsymbol\omega}=c\boldsymbol\omega$ with $c\neq 0$, and the specific  $c>0$ can be determined by  normalization  \citep{keshavarz2011imputing,englert2017inverse}.

We have noted that existing IOC methods   typically assume that the full trajectory data $\boldsymbol{\xi}$ is available.  Violation of this assumption will lead to a failure of existing approaches,  as we will demonstrate in the numerical evaluation in Section \ref{algo1simulation}. In this article, we aim to address this challenge  by developing a technique to estimate $\boldsymbol{\omega}$ only using incomplete trajectory data. Specifically, given a relevant feature set $\mathcal{F}^*$, 
the goal of this paper  is to achieve a successful estimate  $\hat{\boldsymbol{\omega}}=c\boldsymbol{\omega}$  using an \emph{incomplete trajectory observation}
\begin{equation}\label{incompleteTraj}
\boldsymbol{\xi}_{ {t:t+l}}=\{\boldsymbol{\xi}_k: t\leq k \leq  {t+l}\}\subseteq\boldsymbol{\xi},
\end{equation}
which is a segment of $\boldsymbol{\xi}$ within the time interval  $[t, {t+l}]\subseteq[0,T]$.
Here, $t$ is called the \emph{observation starting time} and $l=1,2,\cdots$ called the \emph{observation length}, with $0\leq t<  {t+l}\leq T$. Moreover, for any  observation starting time $t$, we aim to efficiently find the minimal required observation, that is, $l_{\min}$, to achieve a successful estimate of $\boldsymbol{\omega}$. Note that in the above  problem setting,  we \emph{only know} that the  data  $\boldsymbol{\xi}_{t: {t+l}}$ is a   segment  of a system  trajectory $\boldsymbol{\xi}$; we do not know the value of $t$ (i.e., the observation starting time relative to the start of the trajectory), and   do not require  knowledge of any other information about $\boldsymbol{\xi}$ such as the time horizon $T$  {or which type of  optimal control problem $\boldsymbol{\xi}$ is a solution to.}

\section{The Recovery Matrix}
In this section, we introduce the key concept of the \emph{recovery matrix}
and show its relation to  IOC process.
Some properties of the recovery matrix are  investigated to provide    insights to  IOC process.
Connections between the recovery matrix and existing methods is also discussed. The implementation of the recovery matrix is finally presented.

\subsection{Definition of the Recovery Matrix} \label{definitionrecoverymatrix}
We first present the  definition of the recovery matrix, then show its relation to the IOC problem solution, which is also the motivation of  the recovery matrix.
\begin{definition}\label{def_rm}
	Let  a segment of the trajectory,  $\boldsymbol{\xi}_{t: {t+l}}\subseteq \boldsymbol{\xi}$ in (\ref{incompleteTraj}),
	and a  candidate feature set  $\mathcal{F}=\{\phi_1,\phi_2,\cdots,\phi_r\}$ be given. Let $\boldsymbol\phi=\col \mathcal{F}$. Then the recovery matrix, denoted by $\boldsymbol{H}(t,l)$, is defined as:
\begin{equation}\label{H}
\boldsymbol{H}(t,l)=
\begin{bmatrix}
\boldsymbol{H}_1(t,l) & \boldsymbol{H}_2(t,l)
\end{bmatrix}\in \mathbb{R}^{ml\times (r+n)},
\end{equation}	
 with 
	\begin{align}
	\boldsymbol{H}_1(t,l)&=
	\boldsymbol{F}_u(t,l)\boldsymbol{F}^{-1}_x(t,l)\boldsymbol{\Phi}_x(t,l)+\boldsymbol{\Phi}_u(t,l), \label{H1}\\
	\boldsymbol{H}_2(t,l)&=
	\boldsymbol{F}_u(t,l)\boldsymbol{F}_x^{-1}(t,l)\boldsymbol{V}(t,l). \label{H2}
	\end{align}
	Here, $\boldsymbol{F}_x(t,l)$, $\boldsymbol{F}_u(t,l)$, $\boldsymbol{\Phi}_x(t,l)$, $\boldsymbol{\Phi}_u(t,l)$ and $\boldsymbol{V}(t,l)$ are  defined as
	\begin{align}
	\boldsymbol{F}_x(t,l)&=
	\begin{bmatrix}
	\boldsymbol{I}&\,\,\frac{-\partial \boldsymbol{f}^{\prime}}{\partial \boldsymbol{x}^*_{ {t+1}}} & & \\
	\boldsymbol{0}&\boldsymbol{I} & \ddots&  \\
	&  &\ddots &\,\, \,\,\,\frac{-\partial \boldsymbol{f}^{\prime}}{\partial \boldsymbol{x}^*_{ {t\text{+}l\text{-}1}}}    \\[7pt]
	&  & & \boldsymbol{I} \\
	\end{bmatrix}\in \mathbb{R}^{nl \times nl}\label{HFx},\\
	\boldsymbol{F}_u(t,l)&=
	\begin{bmatrix}
	\frac{\partial \boldsymbol{f}^{\prime}}{\partial \boldsymbol{u}^*_{t}} & & & \\
	&\frac{\partial \boldsymbol{f}^{\prime}}{\partial \boldsymbol{u}^*_{t+1}}& & \\
	& &\ddots  & \\
	& & & \frac{\partial \boldsymbol{f}^{\prime}}{\partial \boldsymbol{u}^*_{t\text{+}l\text{-}1}}
	\end{bmatrix}\in \mathbb{R}^{ml \times nl} \label{HFu},\\
	\boldsymbol{\Phi}_x(t,l)&=
	\begin{bmatrix}
	\frac{\partial \boldsymbol\phi}{\partial \boldsymbol{x}^*_{ {t+1}}} &
	\frac{\partial \boldsymbol\phi}{\partial \boldsymbol{x}^*_{ {t\text{+}2}}} &
	\cdots &
	\frac{\partial \boldsymbol\phi}{\partial \boldsymbol{x}^*_{ {t\text{+}l}}}
	\end{bmatrix}^\prime\in \mathbb{R}^{nl\times r}   \label{HPhix},\\
	\boldsymbol{\Phi}_u(t,l)&=
	\begin{bmatrix}
	\frac{\partial \boldsymbol\phi}{\partial \boldsymbol{u}^*_{t}} &
	\frac{\partial \boldsymbol\phi}{\partial \boldsymbol{u}^*_{t+1}} &
	\cdots &
	\frac{\partial \boldsymbol\phi}{\partial \boldsymbol{u}^*_{t\text{+}l\text{-}1}} 
	\end{bmatrix}^\prime \in \mathbb{R}^{ml\times r} \label{HPhiu},\\
	\boldsymbol{V}(t,l)&=
	\begin{bmatrix}
	\boldsymbol{0}&
	\frac{\partial \boldsymbol{f} }{\partial \boldsymbol{x}^*_{ {t\text{+}l}}}
	\end{bmatrix}^{\prime} \in \mathbb{R}^{nl\times n}, \label{Hv}
	\end{align}
	respectively.
\end{definition}

Before showing a relationship between the above recovery matrix and IOC, we impose the following assumption on the given  candidate feature set $\mathcal{F}$ in Definition \ref{def_rm}.
\begin{assumption}\label{featureassumption}
		 {In  Definition~\ref{def_rm}, the candidate feature set $\mathcal{F}$  contains  as a subset the relevant features $\mathcal{F^*}$ in (\ref{relevantfeatureset}),} i.e., $\mathcal{F^*}\subseteq\mathcal{F}$.
\end{assumption}
\noindent
Assumption \ref{featureassumption} requires that the relevant features $\mathcal{F^*}$ in (\ref{relevantfeatureset}) are contained by the given candidate feature set $\mathcal{F}$, which means that $\mathcal{F}$  also allows for including additional features  that are irrelevant to the optimal control system.  Although restrictive for  choice of features, this assumption is likely to be fulfilled in implementation by providing a larger  set including many features  when the knowledge of exact relevant features is not available. Under Assumption \ref{featureassumption},  without loss of generality, we let
\begin{equation} \label{rm_features}
\mathcal{F}=\{\phi_1^*,\phi_2^*,\cdots,\phi_s^*,\tilde{\phi}_{s+1}, \cdots,\tilde{\phi}_{r}\},
\end{equation}
that is, the first $s$ elements are from $\mathcal{F}^*$ in (\ref{relevantfeatureset}). Then we have
\begin{equation} \label{rm_featurevector}
\boldsymbol\phi(\boldsymbol{x},\boldsymbol{u})=\col {\mathcal{F}}=
\begin{bmatrix}
\boldsymbol\phi^*(\boldsymbol{x},\boldsymbol{u}) \\
\tilde{\boldsymbol{\phi}} (\boldsymbol{x},\boldsymbol{u})
\end{bmatrix} \in \mathbb{R}^r,
\end{equation} where $\boldsymbol\phi^*\in\mathbb{R}^s$ are the relevant feature vector in (\ref{costfunction}) while   $\tilde{\boldsymbol{\phi}}\in\mathbb{R}^{(r-s)}$ corresponds to the features that are not in $\mathcal{F^*}$. We define a weight vector 
{\begin{equation}\label{rm_weightvector}
	\bar{\boldsymbol{\omega}}=\col\{\boldsymbol{\omega},\boldsymbol{0}\} \in \mathbb{R}^r
	\end{equation}}%
corresponding to (\ref{rm_featurevector}), where $\boldsymbol{\omega}$ are the  weights in (\ref{costfunction}) for $\boldsymbol{\phi}^*$. Based on
(\ref{costfunction}), we can say that the system's optimal trajectory $ \boldsymbol{\xi}$ in (\ref{traj}) also
(locally) minimize the cost function of
\begin{equation} \label{costfun_ex}
{J}(\boldsymbol{\xi})=\sum_{ {k=0}}^{T} \bar{\boldsymbol{\omega}}^{\prime}\boldsymbol{\phi}(\boldsymbol x_k,\boldsymbol u_k),
\end{equation}
with the dynamics and initial condition in (\ref{dynamics}).
 {
Next, we will distinguish the three optimal control settings, as described in the problem formulation, and then establish the relationship between the   recovery matrix and the IOC problem solution.
}

 {\subsubsection*{Case I: Finite-Horizon Free-End Optimal Control.} We first consider the optimal control setting with finite horizon $T$ and free final state $\boldsymbol{x}_{T+1}$. In this case, given the cost function (\ref{costfun_ex}) and the dynamics constraint (\ref{dynamics}), one can define the following Lagrangian:
	\begin{equation} \label{lagrange_ex}
	{L}={J}(\boldsymbol{\xi}){+}\sum_{k=0}^{T}\boldsymbol \lambda_{k{+}1}^{\prime}\big(\boldsymbol{f}(\boldsymbol x_{k},\boldsymbol u_k)-\boldsymbol x_{k{+}1}\big),
	\end{equation}
}\noindent
where  $\boldsymbol \lambda_{k+1}\in {\mathbb{R}^n}$, $k=0,1,\dots,  {T}$, is   Lagrange multipliers. According to the Karush-Kuhn-Tucker (KKT)  optimality conditions \citep{boyd2004convex}, there exist multipliers $\boldsymbol\lambda^*_{1:T+1}=\col\{\boldsymbol \lambda_1^*, \boldsymbol \lambda_{2}^*,...,\boldsymbol \lambda_{T}^*,  \boldsymbol{\lambda}_{T+1}^* \}$, also referred to as costates, such that the optimal trajectory $\boldsymbol{\xi}$   must satisfy the following  conditions
\begin{subequations}\label{KKT_condition}
	\begin{align}
	\frac{\partial {L}}{\partial \boldsymbol{x}^*_{ {1:T+1}}}&=\boldsymbol{0}, \label{KKT_condition.1} \\
	\frac{\partial {L}}{\partial \boldsymbol{u}^*_{ {0:T}}}&=\boldsymbol{0}. \label{KKT_condition.2}
	\end{align}
\end{subequations} 
Based on the definitions in (\ref{HFx})-(\ref{Hv}),  the  equations in (\ref{KKT_condition.1}) and (\ref{KKT_condition.2}) can be written as
{\begin{subequations}\label{kktmat}
		\begin{align}
		-\boldsymbol{F}_x( {0,T})\boldsymbol{\lambda}_{ {1:T}}^*{+}\boldsymbol{\Phi}_x( {0,T})\bar{\boldsymbol \omega}&=\boldsymbol{0}=-\boldsymbol{V}( {0,T})\boldsymbol{\lambda}^*_{ {T+1}},
		\label{kktmat:1}\\
		\boldsymbol{F}_u( {0,T})\boldsymbol{\lambda}_{ {1:T}}^*{+}\boldsymbol{\Phi}_u( {0,T})\bar{\boldsymbol \omega}&=\boldsymbol{0},
		\label{kktmat:2}
		\end{align}
\end{subequations}}%
respectively,   {where in (\ref{kktmat:1}),  $\boldsymbol{\lambda}^*_{T+1}=\boldsymbol{0}$  directly results from extending  (\ref{KKT_condition.1}) at the final state $\boldsymbol{x}_{T+1}$.} The optimality equations in  (\ref{kktmat}) are established for  full optimal trajectory  $ \boldsymbol{\xi}$. Given any segment  of the trajectory, say $\boldsymbol{\xi}_{ {t:t+l}}\subseteq \boldsymbol{\xi}$ in (\ref{incompleteTraj}), the following equations can be obtained  by  partitioning (\ref{kktmat:1}) and (\ref{kktmat:2}) in rows, 
{\begin{subequations}\label{kktmatin}
		\begin{align}
		-\boldsymbol{F}_x(t,l)\boldsymbol{\lambda}_{ {t+1:t+l}}^*+\boldsymbol{\Phi}_x(t,l)\bar{\boldsymbol \omega}&=-\boldsymbol{V}(t,l)\boldsymbol{\lambda}_{ {t+l+1}}^*,
		\label{kktmatin:1}\\
		\boldsymbol{F}_u(t,l)\boldsymbol{\lambda}_{ {t+1:t+l}}^*+\boldsymbol{\Phi}_u(t,l)\bar{\boldsymbol \omega}&=\boldsymbol{0},
		\label{kktmatin:2}
		\end{align}
\end{subequations}}%
respectively.  {For the above (\ref{kktmatin}), we note that when  $\boldsymbol{\xi}_{t:t+l}=\boldsymbol{\xi}_{0:T}$, i.e., when the observation  is the  full trajectory  data $\boldsymbol{\xi}$,  (\ref{kktmatin}) will become (\ref{kktmat}).  Thus, a full trajectory observation can be viewed as a special case of an incomplete trajectory observation, and we will further  discuss this in Section \ref{relationshiptoexistingwork}.}

 {\subsubsection*{Case II: Finite-Horizon Fixed-End Optimal Control.} We next consider the optimal control setting  with a finite  horizon $T$ and a given fixed final state $\boldsymbol{x}_{T+1}=\boldsymbol{x}_{\text{goal}}$. Given the cost function (\ref{costfun_ex}), the dynamics (\ref{dynamics}), and the final state constraint $\boldsymbol{x}_{T+1}=\boldsymbol{x}_{\text{goal}}$,  one can define the following Lagrangian:
		\begin{equation} \label{lagrange_ex2}
	{L}{=}{J}(\boldsymbol{\xi}){+}\sum_{k=0}^{T}\boldsymbol \lambda_{k\text{+}1}^{\prime}\big(\boldsymbol{f}(\boldsymbol x_{k},\boldsymbol u_k){-}\boldsymbol x_{k\text{+}1}\big){+}\boldsymbol \lambda_{\text{goal}}^{\prime}(\boldsymbol{x}_{T\text{+}1}{-}\boldsymbol{x}_{\text{goal}}),
	\end{equation}
	where the  difference from (\ref{lagrange_ex}) is that the term  $\boldsymbol \lambda_{\text{goal}}^{\prime}(\boldsymbol{x}_{T+1}-\boldsymbol{x}_{\text{goal}})$ is added since the final state is subject to the given $\boldsymbol{x}_{\text{goal}}$  constraint, and $\boldsymbol \lambda_{\text{goal }}\in\mathbb{R}^n$ is the associated Lagrangian multiplier. Following a similar derivation as in Case I, one obtains the same equations in (\ref{kktmatin}) for  any segment data of the trajectory $\boldsymbol{\xi}_{ {t:t+l}}\subseteq \boldsymbol{\xi}$. Here, the only difference from Case I is that when $\boldsymbol{\xi}_{t:t+l}=\boldsymbol{\xi}_{0:T}$, one usually has $\boldsymbol{\lambda}^*_{T+1}=\boldsymbol{\lambda}^*_{\text{goal}}\neq\boldsymbol{0}$ in this case due to the fixed final state constraint, while $\boldsymbol{\lambda}^*_{T+1}=\boldsymbol{0}$ in Case I. In addition, for the finite-horizon optimal control, in which the final state $\boldsymbol{x}_{T+1}$ is penalized using a final cost term added to (\ref{costfunction}),  we can derive the similar result  of $\boldsymbol{\lambda}^*_{T+1}\neq\boldsymbol{0}$.
}

 {\subsubsection*{Case III: Infinite-Horizon Optimal Control.}  For the infinite-horizon optimal control setting, the  optimal trajectory  $\boldsymbol{\xi}$ is more conveniently characterized by the  Bellman optimality condition \citep{bertsekas1995dynamic}:
\begin{equation}\label{bellmanequ}
V(\boldsymbol{x}^*_k)=\bar{\boldsymbol{\omega}}^{\prime}\boldsymbol{\phi}(\boldsymbol x_k^*,\boldsymbol u_k^*)+V(\boldsymbol{f}(\boldsymbol{x}_k^*,\boldsymbol{u}_k^*)),
\end{equation}
where $V(\boldsymbol{x}^*_k)$ is the (unknown) optimal cost-to-go function evaluated at   state $\boldsymbol{x}_k^*$. 
Next, we differentiate the Bellman optimality equation in (\ref{bellmanequ}) on both sides with respect to $\boldsymbol{x}_k^*$ while denoting $\boldsymbol{\lambda}_{k}^*=\frac{\partial V(\boldsymbol{x}_{k})}{\partial\boldsymbol{x}_{k}^*}\in\mathbb{R}^{n}$, and then  obtain
\begin{equation}\label{bellman.x}
\boldsymbol{\lambda}_{k}^*=\small{\frac{\partial\boldsymbol{\bar{\phi}}^\prime}{\partial \boldsymbol{x}_k^*}}\boldsymbol{\bar{\omega}}+\small{\frac{\partial\boldsymbol{{f}}^\prime}{\partial \boldsymbol{x}_k^*}}\boldsymbol{\lambda}^*_{k+1}.
\end{equation}
Differentiating the Bellman optimality equation  (\ref{bellmanequ})  on both sides with respect to $\boldsymbol{u}_k^*$ yields
\begin{equation}\label{bellman.u}
\boldsymbol{0}=\small{\frac{\partial\boldsymbol{\bar{\phi}}^\prime}{\partial \boldsymbol{u}_k^*}}\boldsymbol{\bar{\omega}}+\small{\frac{\partial\boldsymbol{{f}}^\prime}{\partial \boldsymbol{u}_k^*}}\boldsymbol{\lambda}^*_{k+1}.
\end{equation}
For any available trajectory segment $\boldsymbol{\xi}_{ {t:t+l}}\subseteq \boldsymbol{\xi}$, we stack  equation (\ref{bellman.x})  for  all $\boldsymbol{x}^*_{t+1:t+l}$ and stack equation (\ref{bellman.u}) for all $\boldsymbol{u}^*_{t:t+l-1}$, and  obtain the same equations in (\ref{kktmatin}). 
}

\bigskip
\noindent
 {From the above analysis, we conclude  that, for any   trajectory segment  $\boldsymbol{\xi}_{ {t:t+l}}\subseteq \boldsymbol{\xi}$, regardless of the corresponding optimal control problem, we can always use  the segment data  $\boldsymbol{\xi}_{ {t:t+l}}$ to establish the equations (\ref{kktmatin}). Thus, in what follows, we do not distinguish the specific optimal control settings, and only focus on  equations (\ref{kktmatin}) to show the relationship between the recovery matrix in Definition \ref{def_rm} and  IOC problem solution.} 

\medskip

By noticing that $\boldsymbol{F}_x(t,l)$ in (\ref{kktmatin:1}) is always invertible, we combine (\ref{kktmatin:1}) with (\ref{kktmatin:2}) and eliminate  $\boldsymbol{\lambda}_{ {t+1:t+l}}^*$, which then yields
{\begin{multline} \label{recoveryequation}
	\Big(\boldsymbol{F}_u(t,l)\boldsymbol{F}_x^{-1}(t,l)\boldsymbol{\Phi}_x(t,l)+\boldsymbol{\Phi}_u(t,l)\Big) \bar{\boldsymbol \omega}\\ 
	+\Big(\boldsymbol{F}_u(t,l)\boldsymbol{F}_x^{-1}(t,l)\boldsymbol{V}(t,l)\Big) \boldsymbol \lambda_{ {t+l+1}}^*=\boldsymbol{0}.
	\end{multline}}%
Considering the definition of the recovery matrix in (\ref{H})-(\ref{H2}), (\ref{recoveryequation}) can be written as
\begin{equation}\label{recoveryequationbyH}
\begin{aligned}
&\boldsymbol{H}_1(t,l)\bar{\boldsymbol\omega}+\boldsymbol{H}_2(t,l) \boldsymbol \lambda_{ {t{+}l{+}1}}^*\\
= \,\,&\boldsymbol{H}(t,l)
\begin{bmatrix}
\bar{\boldsymbol{\omega}}\\
\boldsymbol \lambda_{ {t+l+1}}^*
\end{bmatrix}=\boldsymbol{0}.
\end{aligned}
\end{equation}
Equation (\ref{recoveryequationbyH}) reveals that the weights $\bar{\boldsymbol \omega}$ and costate $ \boldsymbol \lambda_{ {t+l+1}}^*$ must  satisfy a linear equation, where the coefficient matrix is exactly the  recovery matrix that is defined on
the  trajectory segment $\boldsymbol{\xi}_{t:t+l}\subseteq\boldsymbol{\xi}$, and candidate  feature set $\mathcal{F}$. Here,   the costate $\boldsymbol{\lambda}_{ {t+l+1}}^*$  can be interpreted as a variable encoding  the \emph{unseen future information}   beyond the  observational interval $[t,t+l]$.  {In fact, from the discussions for Case III, we note that   costate $\boldsymbol{\lambda}_{t+l+1}^*$ is the  gradient of the optimal cost-to-go function with respect to the state evaluated at $\boldsymbol{{x}}^*_{t+l+1}$.  
}

{In IOC problems, in order to obtain an estimate of the unknown   weights $\bar{\boldsymbol{\omega}}$ only using the available segment data $\boldsymbol{\xi}_{t:t+l}$, one also needs to account for the unknown  $\boldsymbol{\lambda}^*_{ {t+l+1}}$, as in (\ref{recoveryequationbyH}).
The following  theorem establishes a relationship between  a trajectory segment $\boldsymbol{\xi}_{t:t+l}\subseteq{\boldsymbol{\xi}}$ and  a successful estimate of the weights $\bar{\boldsymbol{\omega}}$ for given candidate features $\mathcal{F}$. }

\begin{theorem}\label{theorem_rm}
	Given a trajectory segment $\boldsymbol{\xi}_{t: {t+l}}\subseteq{\boldsymbol{\xi}}$, let the recovery matrix $\boldsymbol{H}(t,l)$ be defined as in Definition {\ref{def_rm}}  with the candidate feature set $\mathcal{F}$ satisfying Assumption \ref{featureassumption}. Let a vector   $\col\{\hat{\boldsymbol{\omega}},\hat{\boldsymbol{\lambda}}\}\neq \boldsymbol 0$ satisfy $\col\{\hat{\boldsymbol{\omega}},\hat{\boldsymbol{\lambda}}\}\in\ker \boldsymbol{H}(t,l)$ with $\hat{\boldsymbol{\omega}}\in \mathbb{R}^r$. If
	\begin{equation} \label{theorem_rm_condition}
	 \rank \boldsymbol{{H}}(t,l)= r+n-1,
	\end{equation}
	then there exists a constant  $c\neq0$ such that the $i$th entry of $\hat{\boldsymbol{\omega}}$ satisfies
	\begin{equation} \label{theorem_H_result}
	\hat\omega_i =
	\begin{cases}
	\, c\omega_i, & \text{if} \quad \phi_i\in \mathcal{F}^* \\
	\, 0, & \text{otherwise}
	\end{cases},
	\end{equation}
	and  vector $\col\{\hat{\omega}_i:\phi_i\in \mathcal{F}^*,  {i=1,2,\cdots,r} \}=c\boldsymbol\omega$ thus is a successful estimate of $\boldsymbol \omega$ in (\ref{costfunction}).
\end{theorem}
\begin{proof}
	Based on  equations (\ref{kktmatin}), we note that 
	for a trajectory segment $\boldsymbol{\xi}_{t:t+l}\subseteq\boldsymbol{\xi}$, there always exists  $\boldsymbol{\lambda}^*_{ {t+l+1}}\in\mathbb{R}^n$   such that  $\col\{\bar{\boldsymbol{\omega}},{\boldsymbol{\lambda}^*_{ {t+l+1}}}\}$  satisfies
	 (\ref{recoveryequationbyH}), i.e., $\col\{\bar{\boldsymbol{\omega}},{\boldsymbol{\lambda}^*_{ {t+l+1}}}\}\in\ker \boldsymbol{H}(t,l)$. Due to  (\ref{theorem_rm_condition}) which means that the kernel of $\boldsymbol{H}(t,l)$  is one-dimensional,  any nonzero  vector $\col\{\hat{\boldsymbol{\omega}},\hat{\boldsymbol{\lambda}}\}\in\ker \boldsymbol{H}(t,l)$  will have 
	 $\hat{\boldsymbol{\omega}}=c\bar{\boldsymbol\omega}$ ($c\neq 0$).
	 Thus, one can conclude that  $\hat{\boldsymbol{\omega}}$ is a scaled version of $\bar{{\boldsymbol{\omega}}}$, and that the entries in $\hat{\boldsymbol{\omega}}$ corresponding to the relevant features in $\mathcal{F}^*$ will stack  a successful estimate of $\boldsymbol \omega$  (\ref{costfunction}). This completes the proof. $\qed$
\end{proof}
\begin{remark}
	Theorem \ref{theorem_rm} states that the recovery matrix bridges  trajectory segment data to the unknown objective function.    
	First, the rank of the recovery matrix $\boldsymbol{H}(t,l)$ indicates whether one is able to use the  trajectory segment    $\boldsymbol{\xi}_{t:t+l}\subseteq\boldsymbol{\xi}$ to obtain a successful  estimate of  weights $\bar{\boldsymbol{\omega}}$ for the given candidate features $\mathcal{F}$.
	In particular, if the rank condition  (\ref{theorem_rm_condition}) for the recovery matrix  $\boldsymbol{H}(t,l)$ is satisfied, then any nonzero vector $\col\{\hat{\boldsymbol{\omega}},\hat{\boldsymbol{\lambda}}\}$ in the kernel of  $\boldsymbol{H}(t,l)$  has that: the vector of the first $r$ entries in $\col\{\hat{\boldsymbol{\omega}},\hat{\boldsymbol{\lambda}}\}$, i.e.,  $\hat{\boldsymbol{\omega}}$,  satisfies $\hat{\boldsymbol{\omega}}=c\bar{\boldsymbol{\omega}}$. Second, including additional irrelevant features in $\mathcal{F}$ will not influence the weight estimate  for the relevant features, since  the weight estimates  in $\hat{\boldsymbol{\omega}}$ for these irrelevant features will be zeros. We will  demonstrate this  in   numerical experiments in Section \ref{algo1simulation}.  {We will also demonstrate the use of the recovery matrix to solve different optimal control problems later in   numerical experiments.}
\end{remark}

\subsection{Properties of the Recovery Matrix}\label{propertySection}
Since the recovery matrix connects   trajectory segment data to the unknown cost function, we next investigate  the properties of the recovery matrix, which will provide us a better understanding of how the   data and the selected features are incorporated in  IOC process. We first present an iterative formula for the
recovery matrix.

\begin{lemma}[Iterative Property] \label{lemma1}
	For a trajectory  segment  $\boldsymbol{\xi}_{t:t+l}\subset{\boldsymbol{\xi}}$ and  the  subsequent  data point $\boldsymbol{\xi}_{t+l+1}=\{\boldsymbol{x}_{t+l+1}^*,\boldsymbol{u}_{t+l+1}^*\}$, one has
	\begin{align}
	\boldsymbol{H}(t,l+1)&=
	\begin{bmatrix}
	\boldsymbol{H}_1(t,l+1) & \boldsymbol{H}_2(t,l+1)
	\end{bmatrix}  \label{iterH1}\\
	&=\begin{bmatrix}
	\boldsymbol{H}_{1}(t,l) &\boldsymbol{H}_{2}(t,l) \\
	\frac{\partial \boldsymbol \phi^{\prime}}{\partial \boldsymbol{u}^*_{t+l}}&
	\frac{\partial \boldsymbol{f}^{\prime}}{\partial \boldsymbol{u}^*_{t+l}}
	\end{bmatrix}
	\begin{bmatrix}
	\boldsymbol{I} & \boldsymbol 0 \\
	\frac{\partial \boldsymbol \phi^{\prime}}{\partial \boldsymbol{x}^*_{ {t+l+1}}}&
	\frac{\partial \boldsymbol{f}^{\prime}}{\partial \boldsymbol{x}^*_{ {t+l+1}}}
	\end{bmatrix}, \nonumber 
	\end{align}
	with  $\boldsymbol{H}(t,1)$ corresponding to  $\boldsymbol{\xi}_{t:t+1}=(\boldsymbol{x}^*_{ {t:t+1}},\boldsymbol{u}^*_{ {t:t+1}})$:
	\begin{align} \label{iterH0}
	\boldsymbol H(t,1)&=
	\begin{bmatrix}
	\boldsymbol{H}_1(t,1) & \boldsymbol{H}_2(t,1)
	\end{bmatrix} \nonumber \\
	&=
	\begin{bmatrix}
	(\frac{\partial \boldsymbol{f}^{\prime}}{\partial \boldsymbol{u}^*_{t}}\frac{\partial \boldsymbol\phi^{\prime}}{\partial \boldsymbol{x}^*_{ {t+1}}}+
	\frac{\partial \boldsymbol\phi^{\prime}}{\partial \boldsymbol{u}^*_{t}}) & \frac{\partial \boldsymbol{f}^{\prime}}{\partial \boldsymbol{u}^*_{t}}\frac{\partial \boldsymbol{f}^{\prime}}{\partial \boldsymbol{x}^*_{ {t+1}}}\\
	\end{bmatrix}.
	\end{align}
\end{lemma}
\begin{proof}
	Please see Appendix \ref{prooflemma1}.$\qed$
\end{proof}
\noindent
The   iterative property shows that the recovery matrix can be  calculated  by incrementally  integrating each subsequent data point $\boldsymbol{\xi}_{ {t+1+1}}$ into the current recovery matrix $\boldsymbol{H}(t,l)$. {Due to this property, the computation of  matrix inversions in  the recovery matrix in Definition \ref{def_rm} can be avoided.} This property will be used to devise efficient  IOC algorithms   in Section \ref{sectionalgorithm}.

The recovery matrix is defined on two elements: one is the  segment data $\boldsymbol{\xi}_{t:t+l}$ and the other are the selected candidate features $\mathcal{F}$. In what follows, we will show how these two components affect the recovery matrix and further the  IOC process. For data observations, we expect that including more data points into $\boldsymbol{\xi}_{t:t+l}$ may contribute to enabling the successful estimation of the unknown weights. This is implied by the following lemma.

\begin{lemma} [Rank Nondecreasing Property] 	\label{lemma2}
	For a trajectory  segment  $\boldsymbol{\xi}_{ {t:t+l}}\subset{\boldsymbol{\xi}}$ and any  $\mathcal{F}$, one has
	\begin{equation}
	\rank \boldsymbol{{H}}(t,l)\leq \rank\boldsymbol{{H}}(t,l+1),
	\end{equation}
	if the new trajectory  point $\boldsymbol{\xi}_{ {t+l+1}}=(\boldsymbol{x}^*_{ {t+l+1}},\boldsymbol{u}^*_{ {t+l+1}})$ has $\det (\frac{\partial \boldsymbol{f}}{\partial \boldsymbol{x}^*_{ {t+l+1}}})\neq 0$.
\end{lemma}

\begin{proof}
	Please see Appendix \ref{prooflemma2}. $\qed$
\end{proof}
\noindent
We have noted in Theorem \ref{theorem_rm} that the rank of the recovery matrix is related to whether one is able to use segment data $\boldsymbol{\xi}_{ {t:t+l}}$ to achieve a successful estimate of the  weights. Thus the rank of the recovery matrix can be viewed as an indicator of the capability of the available segment data $\boldsymbol{\xi}_{ {t:t+l}}$ to reflect the unknown weights. Lemma~\ref{lemma2} postulates that additional  data, if its  Jacobian matrix of the dynamics is non-singular,  tends to contribute to solving the IOC problem by increasing the rank of the recovery matrix towards satisfying  (\ref{theorem_rm_condition}), or at least will not make a degrading contribution. Further in Section~\ref{rankdiscussions}, we will  analytically and experimentally  demonstrate in which cases the additional observation data can increase the rank of the recovery matrix, and in which cases the additional  data points cannot increase (i.e. maintain the recovery matrix rank).

The next lemma provides a necessary condition for the rank of the recovery matrix if the candidate feature set $\mathcal{F}$ contains as a subset the relevant features $\mathcal{F^*}$, i.e., $\mathcal{F^*}\subseteq\mathcal{F}$.

\begin{lemma} [Rank Upper Bound Property] 	\label{lemma3}
	If Assumption~\ref{featureassumption} holds,  then for  any trajectory segment  $\boldsymbol{\xi}_{t:t+l}\subseteq{\boldsymbol{\xi}}$,
	\begin{equation}\label{rankbound}
	\rank \boldsymbol{{H}}(t,l)\leq r+n-1
	\end{equation}
	always holds. If
	 there exists another   relevant feature subset   ${\mathcal{\widetilde F}}\subseteq \mathcal{F}$ with corresponding weights $\boldsymbol{\widetilde \omega}$, here $\mathcal{\widetilde F}\neq\mathcal{F^*}$ or  $\boldsymbol{\widetilde \omega}\neq c \boldsymbol{\omega}$, then the above inequality   (\ref{rankbound})  holds strictly:
	\begin{equation}\label{rankbound2}
	\rank \boldsymbol{{H}}(t,l)<  r+n-1.
	\end{equation}
\end{lemma}

\begin{proof}
	Please see Appendix \ref{prooflemma3}.$\qed$
\end{proof}
\noindent
 {Lemma \ref{lemma3} states that if a candidate feature set contains as a subset the relevant features under which the system  trajectory $\boldsymbol{\xi}$ is optimal}, the kernel of the recovery matrix (for any  data segment)  is at least one-dimensional. Moreover, when there exist more than one combination of relevant features among the given candidate features,  which means there exists another subset of relevant features or another independent  weight vector,  then the rank condition  (\ref{theorem_rm_condition}) in Theorem \ref{theorem_rm} is \emph{impossible} to be fulfilled for the trajectory segment $\boldsymbol{\xi}_{t:t+l}$ regardless of the observation length $l$ and starting time $t$ (we will also experimentally illustrate this in Section \ref{rankdiscussions}). This also  implies that though  Assumption~\ref{featureassumption}  is likely to be satisfied by using
a larger feature set  that covers all possible features, it may also lead to the non-uniqueness of relevant features. On the other hand, if Assumption \ref{featureassumption} fails to hold, that is, the 
candidate feature set $\mathcal{F}$ does not contain a \emph{complete set} of relevant features,  then, due to the rank non-decreasing property   in Lemma~\ref{lemma2}, the recovery matrix is more likely to have $\rank \boldsymbol{H}(t,l)=r+n$ after increasing the observation length. To sum up, Lemma \ref{lemma3} can be leveraged to investigate whether the selection of candidate features is proper or not.

Combining Lemma \ref{lemma2} and Lemma \ref{lemma3}, we are able to show: under Assumption \ref{featureassumption}, (i) if the rank of the recovery matrix is less than $r+n-1$, then increasing the observation length $l$ to include additional trajectory points  may increase the rank; and (ii) once the segment  reaches $\rank \boldsymbol{{H}}=r+n-1$, additional
observation data will not  increase the rank of the recovery matrix  and the successful estimate of feature weights  can be found in the kernel of the recovery matrix. We will experimentally demonstrate this later in Section \ref{rankpropertysimulation1} and Section~\ref{rankpropertysimulation2}. In Sections~\ref{rankdiscussions}, we will further analyze how additional  observation data will change the rank of the recovery matrix.

\subsection{Relationship with Prior  Work}\label{relationshiptoexistingwork}
We next discuss the relationship between the above recovery matrix  and existing IOC techniques   \citep{keshavarz2011imputing,puydupin2012convex,englert2017inverse,molloy2017finite,johnson2013inverse,molloy2016discrete,johnson2013inverse,aghasadeghi2014inverse}. In those methods, an \emph{observation of the system's full trajectory} $\boldsymbol{\xi}$ is considered,  for which a set of  optimality equations, such as the KKT conditions \citep{boyd2004convex} or Pontryagin's Minimum principle \citep{pontryagin1962mathematical}, is then established. As developed in \citep{englert2017inverse,molloy2017finite}, based on  optimality conditions, a general form for using  trajectory data to establish a linear constraint on the unknown feature weights $\boldsymbol{\omega}$ can be summarized as
\begin{equation}\label{compareequ1}
\boldsymbol{M}(\boldsymbol{\xi}){\boldsymbol{\omega}}=\boldsymbol{0},
\end{equation}
where $\boldsymbol{M}(\boldsymbol{\xi})$ is the coefficient matrix that depends on the  trajectory data $\boldsymbol{\xi}$. An implicit requirement by those methods  is that the  \emph{observed data $\boldsymbol{\xi}$ itself has to  be optimal} with respect to the  cost function, thus    \emph{full trajectory data $\boldsymbol{\xi}_{0:T}$ is generally required} (otherwise, incomplete data $\boldsymbol{\xi}_{t:t+l}\subseteq\boldsymbol{\xi}$ \emph{itself} in general does not  optimize the  cost function).

Conversely,  through the  recovery matrix developed in this paper,  any \emph{trajectory segment} $\boldsymbol{\xi}_{t:t+l}\subseteq\boldsymbol{\xi}$   poses a linear constraint on ${\boldsymbol{\omega}}$ by
\begin{equation}\label{compareequ2}
\boldsymbol{H}(t,l)\begin{bmatrix}
{\boldsymbol{\omega}}\\
\boldsymbol{\lambda}_{t{+}l{+}1}
\end{bmatrix}=\boldsymbol{H}_1(t,l){\boldsymbol\omega}+\boldsymbol{H}_2(t,l) \boldsymbol{\lambda}_{t{+}l{+}1}
=\boldsymbol{0}.
\end{equation}
\noindent
Comparing  (\ref{compareequ1}) with (\ref{compareequ2}), we have the following comments.
\begin{itemize}
	\item[1)] If we consider the segment   $\boldsymbol{\xi}_{t:t+l}=\boldsymbol{\xi}_{0:T}$, i.e., given the full trajectory  $\boldsymbol{\xi}$, then, due to $\boldsymbol \lambda_{T+1}=\boldsymbol{0}$ ( {assuming the end-free optimal control setting}),   (\ref{compareequ2}) becomes
	\begin{equation}\label{compareequ3}
	\boldsymbol{H}_1(0,T){\boldsymbol\omega}=\boldsymbol{0}.
	\end{equation}
	Comparing (\ref{compareequ3}) with (\ref{compareequ1}) we immediately obtain 
	\begin{equation}\label{thesame}
	 \boldsymbol{H}_1(0,T)=\boldsymbol{M}(\boldsymbol{\xi}).
	\end{equation} Thus, the coefficient matrix $\boldsymbol{M}(\boldsymbol{\xi})$ that is commonly used in existing IOC methods can be considered as a special case of the recovery matrix when the available data is the full trajectory, i.e., $\boldsymbol{\xi}_{t:t+l}=\boldsymbol{\xi}_{0:T}$.

	\item[2)] However, as  in (\ref{thesame}), the  coefficient matrix $\boldsymbol{M}(\boldsymbol{\xi})$ only corresponds to the first term of the recovery matrix, i.e., $\boldsymbol{H}_1(t,l)$.  A key difference of the recovery matrix  is its ability to handle any incomplete data $\boldsymbol{\xi}_{t:t+l}\subseteq\boldsymbol{\xi}$. Since the incomplete data $\boldsymbol{\xi}_{t:t+l}$ itself may not be optimal with respect to the objective function  when $t+l<T$, the \emph{unseen future}  information thus must be taken care of  if one wants to  successfully learn the unknown weights $\boldsymbol{\omega}$. As in  (\ref{compareequ2}), the recovery matrix accounts for  such    {unseen} future information via its second term $\boldsymbol{H}_2(t,l)$ and the unknown costate  $\boldsymbol{\lambda}_{t{+}l{+}1}$. As we will demonstrate later in  experiments (Section \ref{comparison}), such a step to account for the \emph{unseen} future information can enable  successful learning of the cost function  using a very small segment  of the full trajectory.  {Also importantly,  as we have analyzed in Section \ref{definitionrecoverymatrix},  such an advantage  enables the recovery matrix  to solve  the IOC problems for \emph{infinite-horizon optimal control systems}. We will also demonstrate this in Section \ref{IOCinfLQR} with a numerical example.}

	\item[3)] In addition to the capability of dealing with incomplete observation data, the recovery matrix can also provide    insights and computational efficiency for solving IOC problems, as presented in  Section~\ref{propertySection}.  Such properties and advantages, however, cannot be achieved using the coefficient matrix $\boldsymbol{M}(\boldsymbol{\xi})$  in existing IOC methods
	\citep{keshavarz2011imputing,puydupin2012convex,englert2017inverse,molloy2017finite,johnson2013inverse,molloy2016discrete,johnson2013inverse,aghasadeghi2014inverse}.
\end{itemize}

\subsection{Implementation of Rank Evaluation}  \label{recoverycondition}
In practice, directly checking the rank of the recovery matrix is challenging  due to (i)  data noise; (ii) near-optimality of demonstrations, i.e., the observed trajectory slightly deviates from the optimal one; and (iii) computational  error. Thus,
one can use  the following strategies to evaluate the rank of the recovery matrix.

\subsubsection*{Normalization.}
When the observed  data is of low magnitude, the recovery matrix may have the entries rather close to zeros, which may affect the matrix rank evaluation due to computing rounding error. Hence we perform a  normalization of the recovery matrix  before verifying its rank,
\begin{equation}\label{Hnorm}
\boldsymbol{\bar{H}}(t,l)= \frac{\boldsymbol{{H}}(t,l)}{\norm{\boldsymbol{{H}}(t,l)}_F},
\end{equation}
where $\norm{\cdot}_F$ is the Frobenius norm and we only consider the recovery matrix that is not a zero matrix. Then
\begin{equation}
\rank \boldsymbol{\bar{H}}(t,l)=\rank \boldsymbol{H}(t,l).
\end{equation}

\subsubsection*{Rank Index.}
Since we are only interested in whether the rank of the recovery matrix satisfies $\rank\boldsymbol{H}(t,l)=r+n-1$, instead of directly investigating the rank,  we choose to look at the singular values of $\boldsymbol{{\bar{H}}}(t,l)$ by introducing the following  rank index
\begin{align}
\kappa(t,l)=\begin{cases}
\, 0, & \text{if} \quad  \sigma_{2}(\boldsymbol{\bar{H}}(t,l))=0,\\
\, {\sigma_{2}(\boldsymbol{\bar{H}})}/{\sigma_{1}(\boldsymbol{\bar{H}})}, & \text{otherwise.}
\end{cases}
\label{rankindex} 
\end{align}
The condition $\rank\boldsymbol{H}(t,l)=r+n-1$ is thus equivalent to $\kappa(t,l)=+\infty$. However, due to  data noise,   $\kappa(t,l)=+\infty$  usually cannot be reached and thus is a finite value (we will demonstrate this in Section \ref{parametersetting}). We thus   pre-set a threshold  $\gamma$ and  verify  
\begin{equation}\label{rankcase1}
\kappa(t,l)\geq\gamma
\end{equation}
 to decide whether  $\rank\boldsymbol{H}(t,l)=r+n-1$ is fulfilled or not. Later in Section \ref{parametersetting}, we will show how  observation data and noise levels influence the rank index  $\kappa(t,l)$, and how to accordingly choose   a proper $\gamma$.

\section{Proposed IOC Approaches}
Using the   recovery matrix, in this section we  develop the IOC techniques  using  incomplete trajectory observations to learn the cost function  formulated in Section \ref{problemformulation}. Furthermore, we will propose an incremental IOC algorithm  by automatically finding the minimal required observation length.

\subsection{IOC using Incomplete Trajectory Observations}
The following corollary states a method to use an observation of the incomplete trajectory to achieve a successful estimate of  weights for given relevant features. 
\begin{corollary}[IOC using Incomplete Trajectory Observations]\label{coro1}
	For the optimal control system in (\ref{dynamics}), given an incomplete trajectory observation $\boldsymbol{\xi}_{t:t+l}\subseteq\boldsymbol{\xi}$ and a relevant feature set $\mathcal{F}=\mathcal{F}^*$ in
	 (\ref{relevantfeatureset}),  the recovery matrix $\boldsymbol{H}(t,l)$ is  defined as in Definition \ref{def_rm}.  	
	If
	\begin{equation} \label{theo2condition}
	\rank \boldsymbol{H}(t,l)=s+n-1,
	\end{equation}
	and  a nonzero vector $\col\{\hat{\boldsymbol{\omega}},\hat{\boldsymbol{\lambda}}\}\in\ker \boldsymbol{H}(t,l)$ with $\hat{\boldsymbol{\omega}}\in \mathbb{R}^s$, then  $\hat{\boldsymbol{\omega}}$ is a successful estimate of $\boldsymbol \omega$, i.e., there must exist a non-zero constant $c$ such that $\hat{\boldsymbol{\omega}}=c\boldsymbol \omega$.
\end{corollary}
\begin{proof}
	Since Corollary \ref{coro1} is a special case of Theorem \ref{theorem_rm}, by following a similar procedure as in proof of Theorem \ref{theorem_rm}, we can show that there must exist a costate $\boldsymbol{\lambda}^*_{t+l+1}$ such that the weight vector $\boldsymbol{\omega}$ in  (\ref{costfunction})  jointly with $\boldsymbol{\lambda}_{t+l+1}^*$ satisfy
	\begin{equation}\label{pftheo2_Hsolution}
	\boldsymbol{H}(t,l)\begin{bmatrix}
	\boldsymbol\omega\\
	\boldsymbol\lambda_{t+l+1}^*
	\end{bmatrix}=\boldsymbol{0}.
	\end{equation}
	Since the rank condition (\ref{theo2condition}) holds, it follows that the nullity of $\boldsymbol{H}(t,l)$ is one. Then for any non-zero vector $\col\{\hat{\boldsymbol{\omega}},\hat{\boldsymbol{\lambda}}\} \in \ker \boldsymbol{H}(t,l)$, there must exist a constant $c\neq 0$ such that $\hat{\boldsymbol{\omega}}=c\boldsymbol\omega$.  $\hat{\boldsymbol{\omega}}$ thus is a successful estimate of $\boldsymbol\omega$, which completes the proof.$\qed$
\end{proof}

\begin{remark}\label{remark2.2}
	Suppose that  in Corollary \ref{coro1},  (\ref{theo2condition}) is not satisfied.  According to Lemma~\ref{lemma3},  $\rank\boldsymbol{H}(t,l)<s+n-1$ holds and thus  dimension of $\ker\boldsymbol{H}(t,l)$ is at least two. This means that 
	another  weight vector,  independent of $\boldsymbol\omega$, could be found in  $\ker\boldsymbol{H}(t,l)$, and the current segment $\boldsymbol{\xi}_{t:t+l}$ may be  generated by this different  weight vector.  In this case,  true  weights  are  not distinguishable or recoverable with respect to  $\boldsymbol{\xi}_{t:t+l}$. The reason for this case can be insufficiently long observation  or low data informativeness, both of  which may be remedied by including additional data  (i.e., increase $l$) according to Lemma \ref{lemma2} (we will illustrate  this later in experiments).
\end{remark}

\subsection{Incremental IOC Algorithm}\label{sectionalgorithm}
Combining Theorem \ref{theorem_rm} and the properties of the recovery matrix, one has the following conclusions: (i) as
 observations of more data points  may contribute to increasing  the rank of the  recovery matrix (Lemma \ref{lemma2}), which is however    bounded from above (Lemma \ref{lemma3}),   thus  the minimal required observation length $l_{\min}$  that reaches the rank upper bound  can be found; (ii) from Lemmas \ref{lemma2} and \ref{lemma3}, the minimal required observation length can  be found even if additional irrelevant features exist; and (iii)  from Lemma \ref{lemma1}, the minimal required observation length can be found efficiently. In sum, we have the following incremental IOC approach.

\begin{corollary} [Incremental IOC Approach] \label{corollary2.2}
	Given  candidate features   $\mathcal{F}$  satisfying Assumption \ref{featureassumption},  the recovery matrix $\boldsymbol{H}(t,l)$, starting from $t$,  is updated at  each time  step with a new observed  point $\boldsymbol{\xi}_{t+l+1}=(\boldsymbol{x}^*_{t+l+1},\boldsymbol{u}^*_{t+l+1})$ via Lemma~\ref{lemma1}. Then the minimal observation length that suffices for a successful estimate of the feature weights is
	\begin{align}\label{minobservations}
	l_{\min}(t)=\min\big\{l\hspace{1pt}|\hspace{1pt}\rank\boldsymbol{H}(t,l)=|\mathcal{F}|+n-1\big\}.
	\end{align} 
	For any nonzero vector  $\col\{\hat{\boldsymbol\omega},\hat{\boldsymbol\lambda}\}\in\ker\ \boldsymbol{H}(t,l_{\min}(t))$ with
	$\hat{\boldsymbol\omega}\in\mathbb{R}^{|\mathcal{F}|}$,   $\hat{\boldsymbol\omega}$ is a successful estimate of the weights for   ${\mathcal{F}}$ with the  weights for irrelevant features being zeros.
\end{corollary}

\begin{proof}
	Corollary \ref{corollary2.2} is a direct application of Theorem \ref{theorem_rm}, Lemma \ref{lemma1}, Lemma \ref{lemma2}, and Lemma \ref{lemma3}.  {$\qed$}
\end{proof}

From  Corollary \ref{corollary2.2}, we note that starting from time $t$, the minimal required observation length $l_{\min}(t)$ to solve IOC problems is the one  satisfying (\ref{minobservations}). As we will show later in experiments  (Sections \ref{rankdiscussions}, \ref{rankpropertysimulation2}, \ref{observationnoise}, and \ref{irrelevantfeatures}), 
$l_{\min}(t)$ varies depending on the  informativeness of   data $\boldsymbol{\xi}_{t:t+l}$ and the selected candidate features. Whatever influences $l_{\min}(t)$, one can always find a necessary lower bound of the minimal required observation length due to  the size of the recovery matrix, $\boldsymbol{H}(t,l)\in\mathbb{R}^{ml\times(|\mathcal{F}|+n)}$, and  matrix rank properties, that is,
\begin{equation}
l_{\min}(t)\geq\ceil[\bigg]{\frac{|\mathcal{F}|+n-1}{m}},
\label{unifomlb}
\end{equation}
where $\lceil{\cdot}\rceil$ is the ceiling operation. (\ref{unifomlb})  implies that including additional irrelevant  features to $\mathcal{F}$ will require more data in order to successfully solve  IOC problems (as shown later in Section \ref{irrelevantfeatures}).

In practice, directly applying Corollary \ref{corollary2.2} is challenging in the presence of  data noise, near-optimality of  trajectory, computing error, etc. Thus we adopt the following strategies for implementation. First, (\ref{minobservations}) can be investigated based on the rank index  (\ref{rankindex}) by checking (\ref{rankcase1}) (the choice of $\gamma$ will be discussed later in Section \ref{parametersetting}). Second, the computation of a successful estimate can be implemented by solving the following constrained optimization
{\begin{align} \label{compute_aug_cost}
	\hat{\boldsymbol\omega} 
	=\arg \min_{\boldsymbol\omega,\boldsymbol{\lambda}}\norm[\bigg]{\bar{\boldsymbol{H}}(t,{l_{\min}}(t))
		\begin{bmatrix}
		\boldsymbol\omega \\
		\boldsymbol\lambda
		\end{bmatrix}}^2,
	\end{align}}%
subject to 
\begin{equation}\label{compute_aug_cost2}
\sum\nolimits_{i=1}^{|\mathcal{F}|}\omega_i=1,
\end{equation}
where $\norm{\cdot}$ denotes the $l_2$ norm and $\bar{\boldsymbol{H}}$ is the normalized recovery matrix (see (\ref{Hnorm})). Here, to avoid trivial solutions, we add the  constraint (\ref{compute_aug_cost2}) to normalize the weight estimate to have sum of one, as used in \citep{englert2017inverse}. 

In sum, the implementation   of the proposed incremental IOC approach in corollary \ref{corollary2.2}  is presented  in Algorithm~\ref{algorithm1}. Algorithm 1 permits  arbitrary observation starting time, and the observation length is automatically found by checking the rank condition using (\ref{rankindex}) and (\ref{rankcase1}). The  algorithm can be viewed as an \emph{adaptive-observation-length} IOC algorithm.

\begin{algorithm2e}[h] 
	\caption{Incremental IOC Algorithm}
	\label{algorithm1}
	\SetKwInput{Initialization}{Initial}
	\KwIn{a candidate feature set $\mathcal{F}$, a  threshold $\gamma$;}

	\Initialization{Any observation starting time $t$; \newline
		 Initialize   $l{=}1$, $\boldsymbol{H}(t,l)$ with $\small\boldsymbol{\xi}_{t:t\text{+}1}{=}(\boldsymbol{x}^*_{t:t\text{+}1},\boldsymbol{u}^*_{t:t\text{+}1})$ 
		  via (\ref{iterH0});
}
	\While{ ${\boldsymbol{H}}(t,l)$   not satisfying (\ref{rankcase1})}{
		Obtain subsequent  data $\boldsymbol{\xi}_{t+l+1}=(\boldsymbol{\boldsymbol{x}}^*_{t+l+1}, \boldsymbol{u}^*_{t+l+1})$\;
		Update $\boldsymbol{H}(t,l)$ with $\boldsymbol{\xi}_{t+l+1}$ via (\ref{iterH1})\;
		$l\leftarrow l+1$\;
	}
	minimal required observation length: ${l_{\min}}(t)=l$ \;
	compute a successful estimate $\hat{\boldsymbol\omega}$ via   (\ref{compute_aug_cost})-(\ref{compute_aug_cost2}).
\end{algorithm2e}

\section{Numerical Experiments} \label{algo1simulation}
We evaluate the proposed  method on two systems. First, on a linear quadratic regulator (LQR) system, we  demonstrate the  rank properties of the recovery matrix,  show  its capability of handling incomplete trajectory data  by  comparing with the related IOC methods, and  {demonstrate its capability to solve IOC for infinite-horizon LQR}. Second, on a simulated two-link robot arm, we evaluate the proposed  techniques in terms of observation noise, including irrelevant features, and parameter settings. Throughout evaluations, we quantify the accuracy of  a weight estimate $\hat{\boldsymbol{\omega}}$ by introducing the following estimation error:
\begin{equation} \label{estimationerror}
e_{\boldsymbol{\omega}}=\inf_{c>0}\frac{\norm{c\hat{\boldsymbol\omega}-{\boldsymbol\omega}}}{\norm{\boldsymbol\omega}},
\end{equation} where $\norm{\cdot}$ denotes the $l_2$ norm, $\hat{\boldsymbol\omega}$ is  the  weight estimate,  and  $\boldsymbol\omega$ is the ground truth. Obviously, $e_{\boldsymbol{\omega}}=0$ means that $\hat{\boldsymbol\omega}$ is a successful estimate of $\boldsymbol{\omega}$. 
The source codes are available at the following link: {\small{\url{ https://github.com/wanxinjin/IOC-from-Incomplete-Trajectory-Observations}}}.

\subsection{Evaluations on LQR Systems}

Consider a  {finite-horizon free-end LQR system} where the  dynamics is
\renewcommand{\arraystretch}{1.2}
\begin{equation}\label{lineardyn}
\boldsymbol x_{k+1}=
\begin{bmatrix}
-1 & 1\\
0 & 1
\end{bmatrix}
\boldsymbol x_{k}+
\begin{bmatrix}
1\\
3
\end{bmatrix}
\boldsymbol u_{ {k}},
\end{equation}
with  initial  $\boldsymbol x_0=[2,-2]^{\prime}$, and  quadratic cost function is 
\begin{align} \label{lqrcost}
J=\sum_{k= {0}}^{T}\big(\boldsymbol x_k^{\prime}Q \boldsymbol x_k+\boldsymbol u_k^{\prime}R \boldsymbol u_{k}\big),
\end{align}
with the time horizon $T=50$. Here, $Q$ and $R$ are positive definite matrices and  assumed to  {have the  structure}
\begin{equation}
Q=\begin{bmatrix}
q_1 &0\\
0&q_2
\end{bmatrix},
\quad
R=r,
\label{lqrcosts}
\end{equation}
respectively. In the feature-weight form  (\ref{costfunction}), the cost function (\ref{lqrcost}) corresponds to the feature vector  $\boldsymbol \phi^*=[x_{1}^2, x_{2}^2, u^2]^{\prime}$ and weights $\boldsymbol \omega=[q_1,q_2,r]^{\prime}$. We here set   $\boldsymbol \omega=[0.1,0.3,0.6]^{\prime}$ to generate the optimal trajectory of the LQR system, which is plotted in Fig.  \ref{figlqr}.  {In IOC problems, we are given the features  $\boldsymbol \phi^*$; the goal is to solve a successful estimate of $\boldsymbol{\omega}$ using the  optimal trajectory data in Fig.~\ref{figlqr}.}

\begin{figure}[h]
	\centering
	\includegraphics[width=0.90\columnwidth]{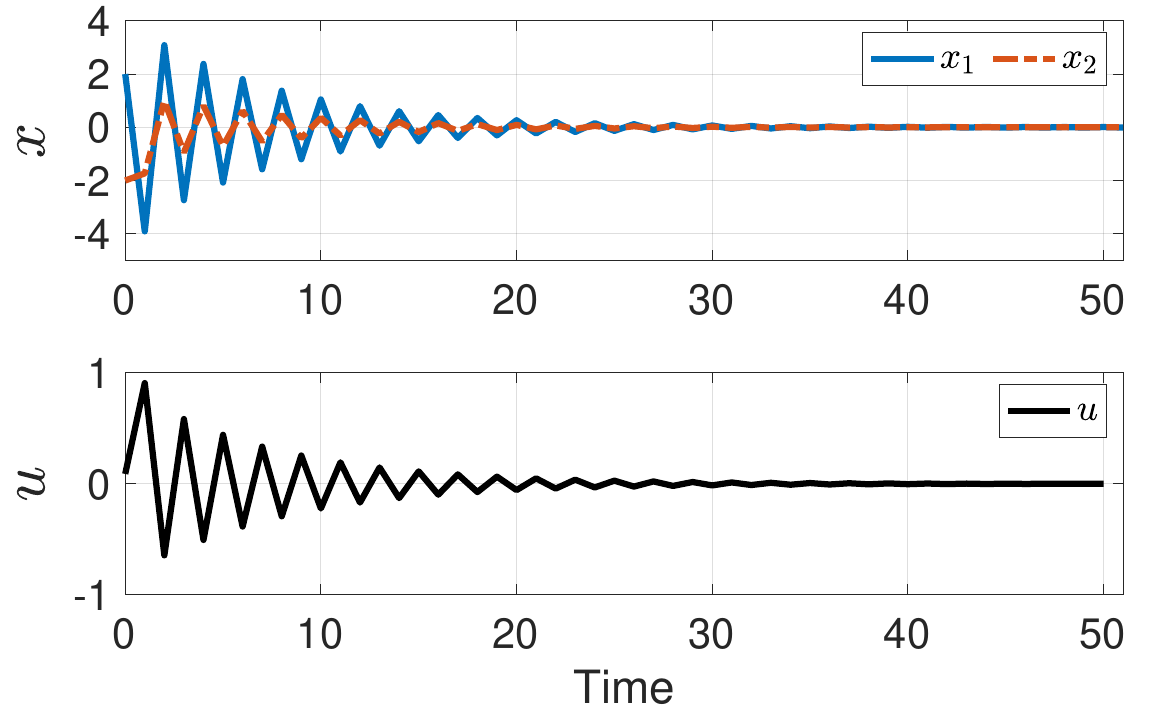}
	\caption{The optimal trajectory   of a LQR system (\ref{lineardyn})-(\ref{lqrcost}) using the weights $\boldsymbol \omega=[0.1,0.3,0.6]^{\prime}$.}
	\label{figlqr}
\end{figure}

\smallskip

\subsubsection{Minimal Required Observations for IOC.}\label{rankpropertysimulation1} Based on the above LQR system, we here illustrate  how the recovery matrix can be used to check whether  incomplete trajectory data suffices for the minimal    observation required for  a successful weight estimation. Given the features $\boldsymbol{\phi^*}=[x_{1}^2, x_{2}^2, u^2]^{\prime}$, we set the  observation  starting time $t=0$, and  incrementally increase the observation length $l$  from $1$ to horizon $T=50$. For each observation length  $l$, we check the rank of the recovery matrix $\boldsymbol{H}(0,l)$ and  solve the weights from the kernel of $\boldsymbol{H}(0,l)$ (the weights are normalized to have sum of one). The results are plotted in Fig. \ref{lqrrank}.

	\begin{figure}[h]
	\centering
	\includegraphics[width=0.9\columnwidth]{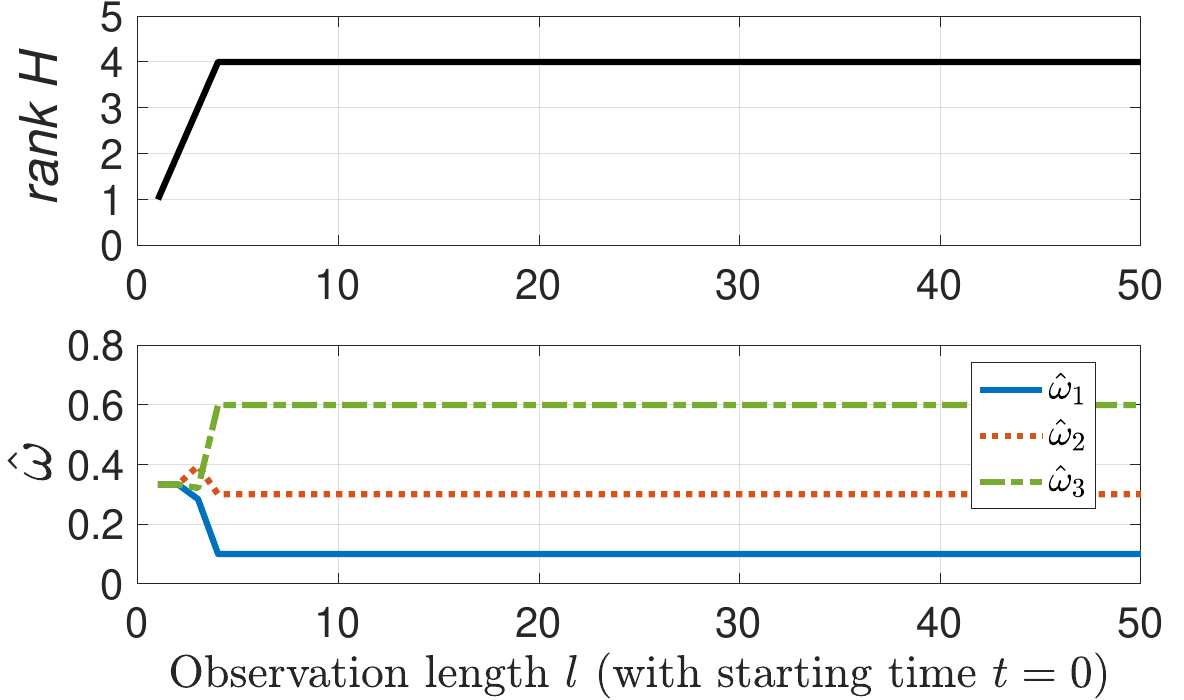}
	\caption{The rank of the recovery matrix and  weight estimate when the observation starts at $t=0$ and  the   observation length $l$ increases from $1$ to $T$. The upper panel shows the rank of the recovery matrix $\boldsymbol{H}(0,l)$ versus $l$; and the bottom panel shows the corresponding weight estimate for each $l$. Note that the given features are $\boldsymbol{\phi^*}=[x_{1}^2, x_{2}^2, u^2]^{\prime}$ and the ground truth weights are $\boldsymbol \omega=[0.1,0.3,0.6]^{\prime}$.  {For $l<4$, since the dimension of the kernel of $\boldsymbol{H}(0,l)$ is at least 2 and thus $\boldsymbol{H}(0,l)[\hat{\boldsymbol{\omega}},\boldsymbol{\lambda}]^\prime=\boldsymbol{0}$ has  multiple solutions of $[\hat{\boldsymbol{\omega}},\boldsymbol{\lambda}]^\prime$,  we  choose the solution $\hat{\boldsymbol\omega}$ from the kernel of $\boldsymbol{H}(0,l)$ randomly.} }
	\label{lqrrank}	
\end{figure}

As shown in the upper panel in Fig. \ref{lqrrank},  including additional   trajectory data points, i.e., increasing the observation length $l$ (from $1$), leads to an increase of the rank of the recovery matrix. When $l=4$ $\rank \boldsymbol{H}(0,l)$ reaches to $4$, which is the rank upper bound $n+r-1=4$, and then $\rank \boldsymbol{H}(0,l)=4$ for all  $l\geq 4$. This illustrates the  properties of the recovery matrix  in Lemma \ref{lemma2} and Lemma \ref{lemma3}.  From the bottom panel in Fig. \ref{lqrrank}, we  see that when $l<4$, for which $\rank \boldsymbol{H}(0,l)<4$, the weight estimate $\hat{\boldsymbol{\omega}}$  is not a successful estimate of $\boldsymbol{\omega}$.  {When $\rank \boldsymbol{H}(0,l)<4$, since the dimension of the kernel of $\boldsymbol{H}(0,l)$ is at least 2 and thus $\boldsymbol{H}(0,l)[\hat{\boldsymbol{\omega}},\boldsymbol{\lambda}]^\prime=\boldsymbol{0}$ has  multiple solutions  $[\hat{\boldsymbol{\omega}},\boldsymbol{\lambda}]^\prime$,  we  choose the solution $\hat{\boldsymbol\omega}$ from the kernel of $\boldsymbol{H}(0,l)$ randomly}. After  $l\geq 4$ when $\rank \boldsymbol{H}(0,l)=4$, the estimate converges to a successful estimate, thus indicating the effectiveness of using the rank condition in (\ref{minobservations}) to  check whether an incomplete observation suffices for the  minimal required observation. 

\begin{figure*}[t]
	\centering
	\begin{subfigure}[b]{0.24\textwidth}
		\centering
		\includegraphics[width=\textwidth]{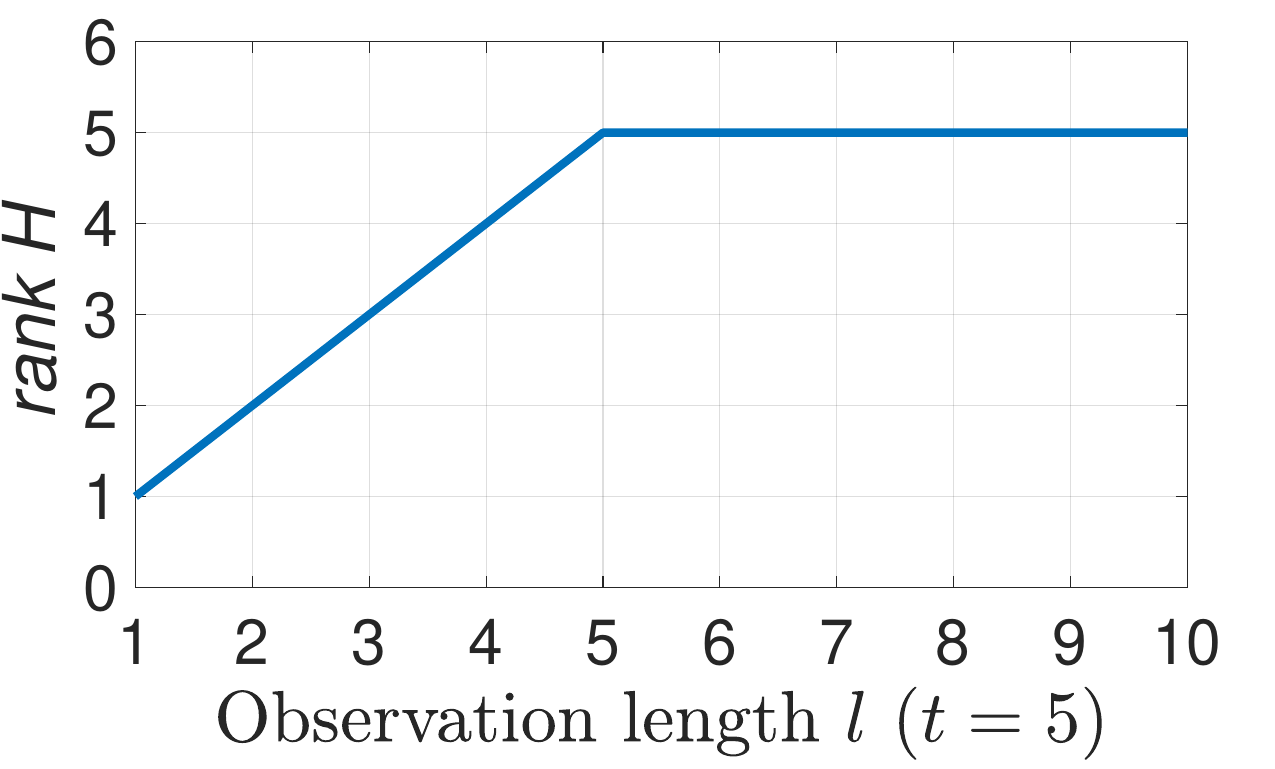}
		\caption{ $t{=}5$, $\mathcal{F}=\{x_1^2,x_2^2,u^2,u^3\}$}
		\label{rankH.1}
	\end{subfigure}
	\hfill
	\begin{subfigure}[b]{0.24\textwidth}
		\centering
		\includegraphics[width=\textwidth]{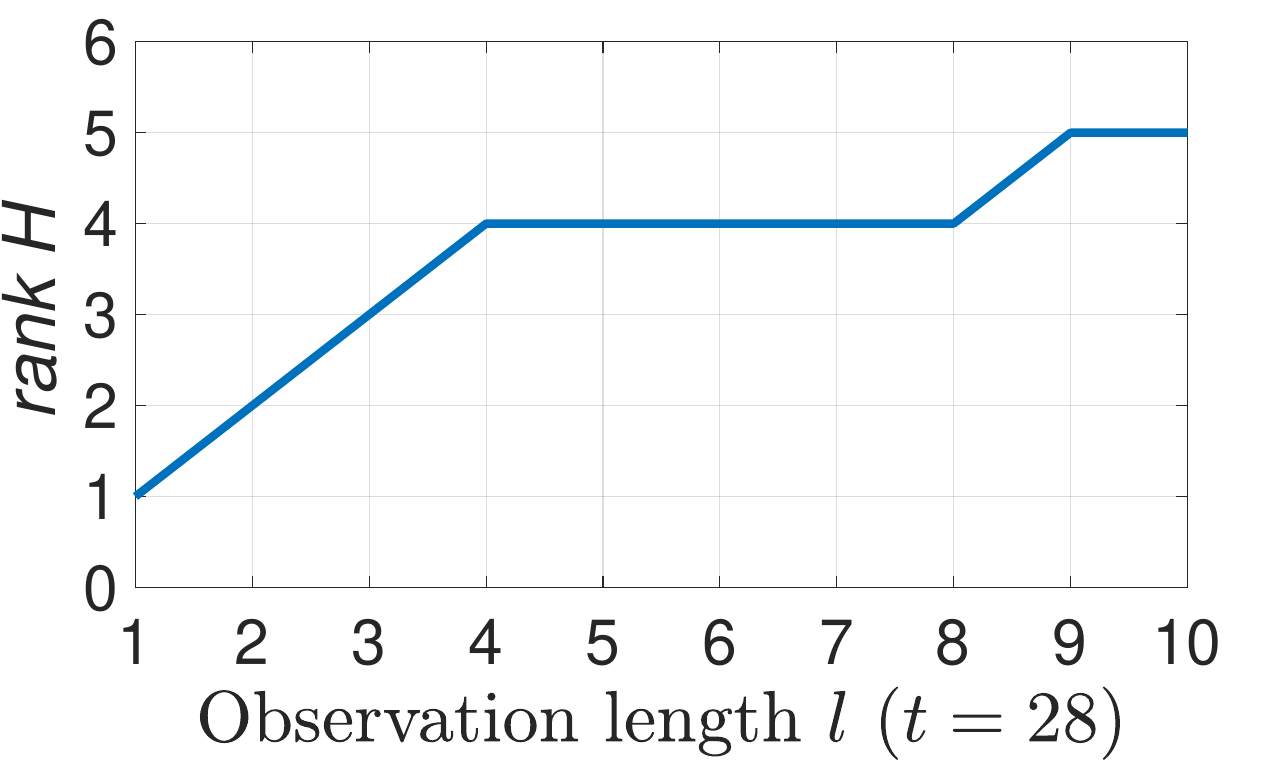}
		\caption{$t{=}28$, $\mathcal{F}=\{x_1^2,x_2^2,u^2,u^3\}$}
		\label{rankH.2}
	\end{subfigure}
	\hfill
	\begin{subfigure}[b]{0.24\textwidth}
		\centering
		\includegraphics[width=\textwidth]{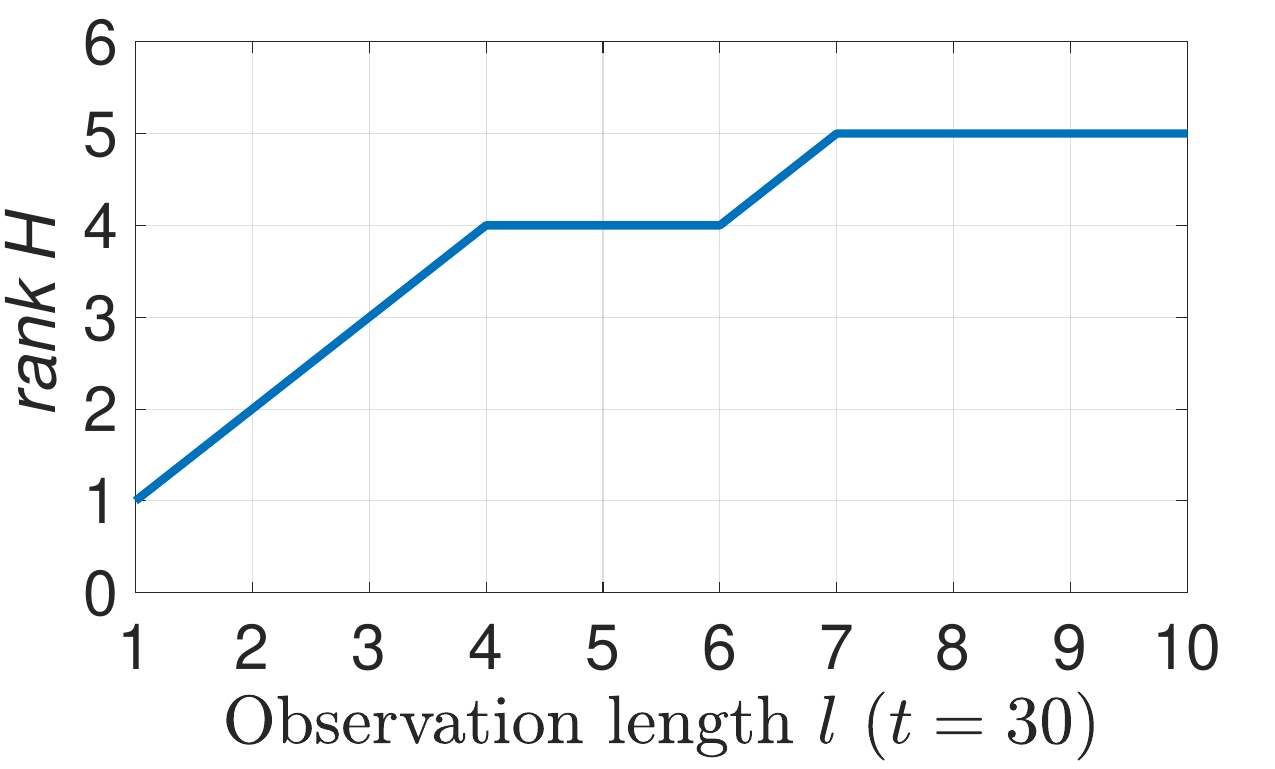}
		\caption{$t{=}30$, $\mathcal{F}=\{x_1^2,x_2^2,u^2,u^3\}$}
		\label{rankH.3}
	\end{subfigure}
	\hfill
	\begin{subfigure}[b]{0.24\textwidth}
		\centering
		\includegraphics[width=\textwidth]{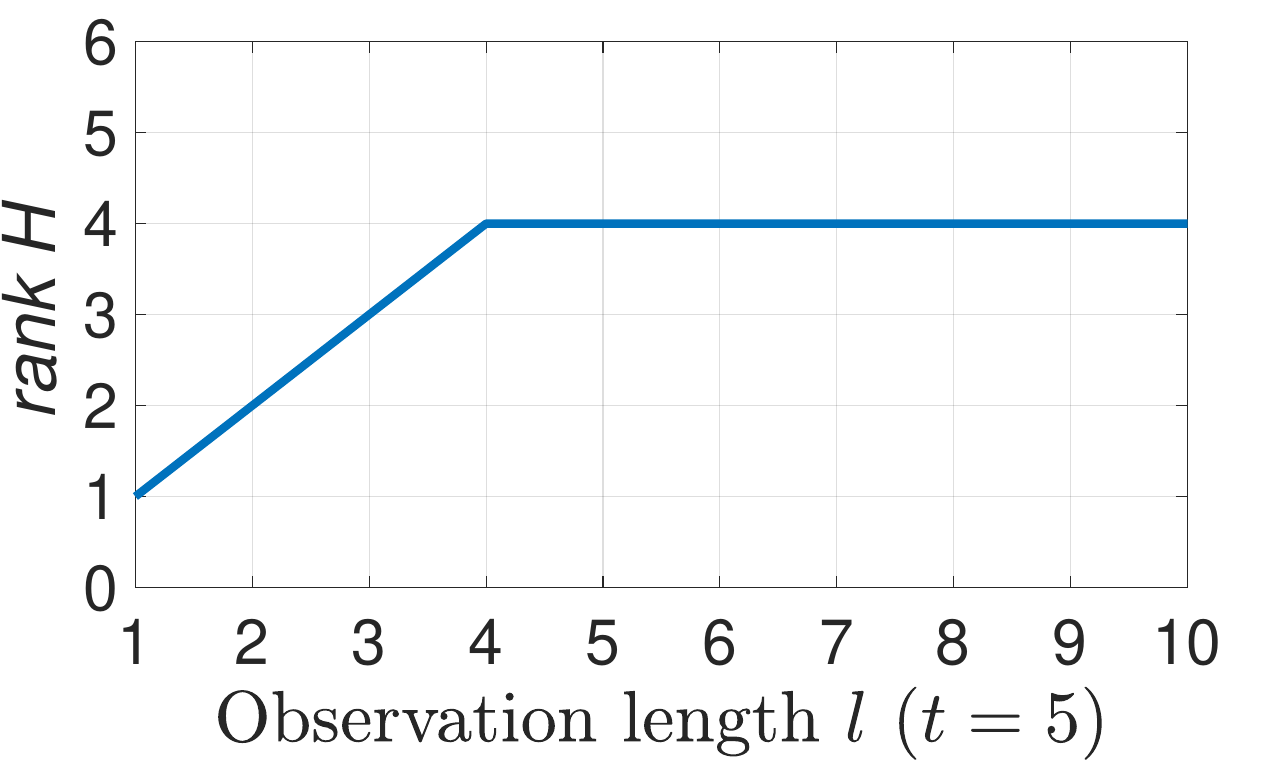}
		\caption{$t{=}5$, $\mathcal{F}=\{x_1^2,x_2^2,u^2,2u^2\}$}
		\label{rankH.4}
	\end{subfigure}
	\caption{The rank of the recovery matrix versus the observation length $l$. For (a), (b), and (c), the observation starting time is at $t=5$, $t=28$, and $t=30$, respectively, and the given candidate feature set is $\mathcal{F}=\{x_1^2,x_2^2,u^2,u^3\}$. For (d), the observation starting time is at $t=5$ and the given candidate feature set is $\mathcal{F}=\{x_1^2,x_2^2,u^2,2u^2\}$.  {In (d), since $\mathcal{F}$ contains two dependent features: $u^2$ and $2u^2$, thus multiple combinations of these  features can be found in  $\mathcal{F}$ to characterize the optimal trajectory, that is, $\{x_1^2, x_2^2, u^2\}$ and $\{x_1^2, x_2^2, 2u^2\}$, and the rank upper bound according to Lemma \ref{lemma3}  is $\rank \boldsymbol{H}(t,l)< r+n-1=5$ and cannot reach $5$.}}
	\label{rankH}
\end{figure*}

\subsubsection{Recovery Matrix Rank for Additional Observations.\label{rankdiscussions}}

Based on the  LQR system, we next show how additional observations affect the rank of the recovery matrix. Here, we vary the observation starting time $t$ and use different candidate feature sets $\mathcal{F}$, and for each case, we incrementally increase the observation length from $l=1$ while checking the rank of the recovery matrix until the rank  reaches its maximum. The results are presented in Fig. \ref{rankH}. For the first three cases in Fig. \ref{rankH.1}-\ref{rankH.3}, we set the observation starting time at $t=5$, $t=28$, and $t=30$, respectively, and use a candidate feature set $\mathcal{F}=\{x_1^2,x_2^2,u^2,u^3\}$; for the fourth case in Fig. \ref{rankH.4}, we  set the observation starting time at $t=5$ and use a candidate feature set $\mathcal{F}=\{x_1^2,x_2^2,u^2,2u^2\}$. Based on the results, we have the following observations and comments.

		1)  {From Fig. \ref{rankH.1}, \ref{rankH.2}, and \ref{rankH.3}}, we can see that additional observation  (i.e., increasing  observation length $l$)  increases or maintains the rank of the recovery matrix, as stated in Lemma \ref{lemma2}, and that continuously increasing the observation length will lead to the upper bound of the recovery matrix's rank, as stated in Lemma~\ref{lemma3}.
		
		2) Comparing Fig. \ref{rankH.1} with Fig. \ref{rankH.4}, we see that although the number of candidate features for both cases are the same, i.e., $|\mathcal{F}|=r=4$, their corresponding maximum  ranks are different: the case in Fig. \ref{rankH.1}  achieves $\max \rank \boldsymbol{H}=5=r+n-1$ (which is the rank condition  (\ref{theorem_rm_condition}) for a successful estimate),  {while   in Fig. \ref{rankH.4} the rank reaches $\max \rank \boldsymbol{H}=4<r+n-1$. This is because  $\mathcal{F}=\{x_1^2,x_2^2,u^2,2u^2\}$ used in Fig. \ref{rankH.4} contains two dependent features, i.e., $u^2$ and $2u^2$, thus multiple combinations of   features,  e.g., $\{x_1^2, x_2^2, u^2\}$ and $\{x_1^2, x_2^2, 2u^2\}$,  can be found in $\mathcal{F}$ to characterize the optimal trajectory. Based on (\ref{rankbound2}) in Lemma \ref{lemma3}, $\rank \boldsymbol{H}(t,l)\leq4$ for all $t$ and $l$ and  the  condition $\rank \boldsymbol{H}(t,l)=5=r+n-1$ for a successful recovery  will never be  fulfilled.}

		3)  Comparing  Fig. \ref{rankH.2}, Fig. \ref{rankH.3} and  Fig. \ref{rankH.1},  {we  note that} in some cases  additional observations  will not increase the rank of the recovery matrix, e.g., when the observation length is $l=5, 6, 7, 8$ in Fig. \ref{rankH.2} and $l=5,6$ in Fig. \ref{rankH.3}. This can be explained using the following relations:
		\begin{align}
		&\rank \boldsymbol{H}(t,l+1)  \nonumber \\
		&=\rank \begin{bmatrix}
		\boldsymbol{H}_{1}(t,l) &\boldsymbol{H}_{2}(t,l) \\
		\frac{\partial \boldsymbol\phi^{\prime}}{\partial \boldsymbol{u}^*_{t+l}}&
		\frac{\partial \boldsymbol{f}^{\prime}}{\partial \boldsymbol{u}^*_{t+l}}
		\end{bmatrix}
		\begin{bmatrix}
		I & \boldsymbol 0 \\
		\frac{\partial \boldsymbol\phi^{\prime}}{\partial \boldsymbol{x}^*_{t+l+1}}&
		\frac{\partial \boldsymbol{f}^{\prime}}{\partial \boldsymbol{x}^*_{t+l+1}}
		\end{bmatrix} \nonumber\\
		&=\rank \begin{bmatrix}
		\boldsymbol{H}_{1}(t,l) &\boldsymbol{H}_{2}(t,l) \\
		\frac{\partial \boldsymbol\phi^{\prime}}{\partial \boldsymbol{u}^*_{t+l}}&
		\frac{\partial \boldsymbol{f}^{\prime}}{\partial \boldsymbol{u}^*_{t+l}}
		\end{bmatrix}  \nonumber\\
		&\geq \rank \begin{bmatrix}
		\boldsymbol{H}_1(t,l) & \boldsymbol{H}_2(t,l)
		\end{bmatrix} = \rank \boldsymbol{H}(t,l), \nonumber
		\end{align}
		where the first two lines  are directly from (\ref{iterH1}), and  the last two lines are  due to $\det (\frac{\partial \boldsymbol{f}^\prime}{\partial \boldsymbol{x}^*_{t+l+1}})=\det\begin{psmallmatrix}
		-1& 0 \\
		1& 1
		\end{psmallmatrix}\neq 0$ and  matrix rank properties. The above equation says that the new observation   $\boldsymbol{\xi}_{t+l+1}=(\boldsymbol{x}^*_{t+l+1},\boldsymbol{u}^*_{t+l+1})$ is incorporated into the recovery matrix $\boldsymbol{H}(t,l)$  in the form of appending  $m$  row vectors $[\frac{\partial \boldsymbol\phi^{\prime}}{\partial \boldsymbol{u}^*_{t+l}}\,\,
		\frac{\partial \boldsymbol{f}^{\prime}}{\partial \boldsymbol{u}^*_{t+l}}]\in \mathbb{R}^{m\times(r+n)}$  to the bottom of $\boldsymbol{H}(t,l)$. If the new observed data point $\boldsymbol{\xi}_{t+l+1}=(\boldsymbol{x}^*_{t+l+1},\boldsymbol{u}^*_{t+l+1})$ is \emph{non-informative}, in other words, if the  appended rows in $[\frac{\partial \boldsymbol\phi^{\prime}}{\partial \boldsymbol{u}^*_{t+l}}\,\,
		\frac{\partial \boldsymbol{f}^{\prime}}{\partial \boldsymbol{u}^*_{t+l}}]$ are \emph{dependent} on the  row vectors in $\boldsymbol{H}(t,l)$,  then according to the matrix rank properties, one will have $\rank\boldsymbol{H}(t,l)=\rank\boldsymbol{H}(t,l+1)$,   thus the new data $\boldsymbol{\xi}_{t+l+1}$  will not increase the rank of the recovery matrix. Otherwise, if the appended rows in $[\frac{\partial \boldsymbol\phi^{\prime}}{\partial \boldsymbol{u}^*_{t+l}}\,\,
		\frac{\partial \boldsymbol{f}^{\prime}}{\partial \boldsymbol{u}^*_{t+l}}]$  are \emph{independent} of the row vectors in $\boldsymbol{H}(t,l)$, that is, the new observed data  $\boldsymbol{\xi}_{t+l+1}=(\boldsymbol{x}^*_{t+l+1},\boldsymbol{u}^*_{t+l+1})$ is \emph{informative}, then this new  $\boldsymbol{\xi}_{t+l+1}$ will increase the rank of the recovery matrix, i.e., $\rank\boldsymbol{H}(t,l)<\rank\boldsymbol{H}(t,l+1)$.

\subsubsection{Comparison with Prior Work.}\label{comparison}

\begin{figure*}
	\centering
	\begin{subfigure}[b]{0.33\textwidth}
		\centering
		\includegraphics[width=\textwidth]{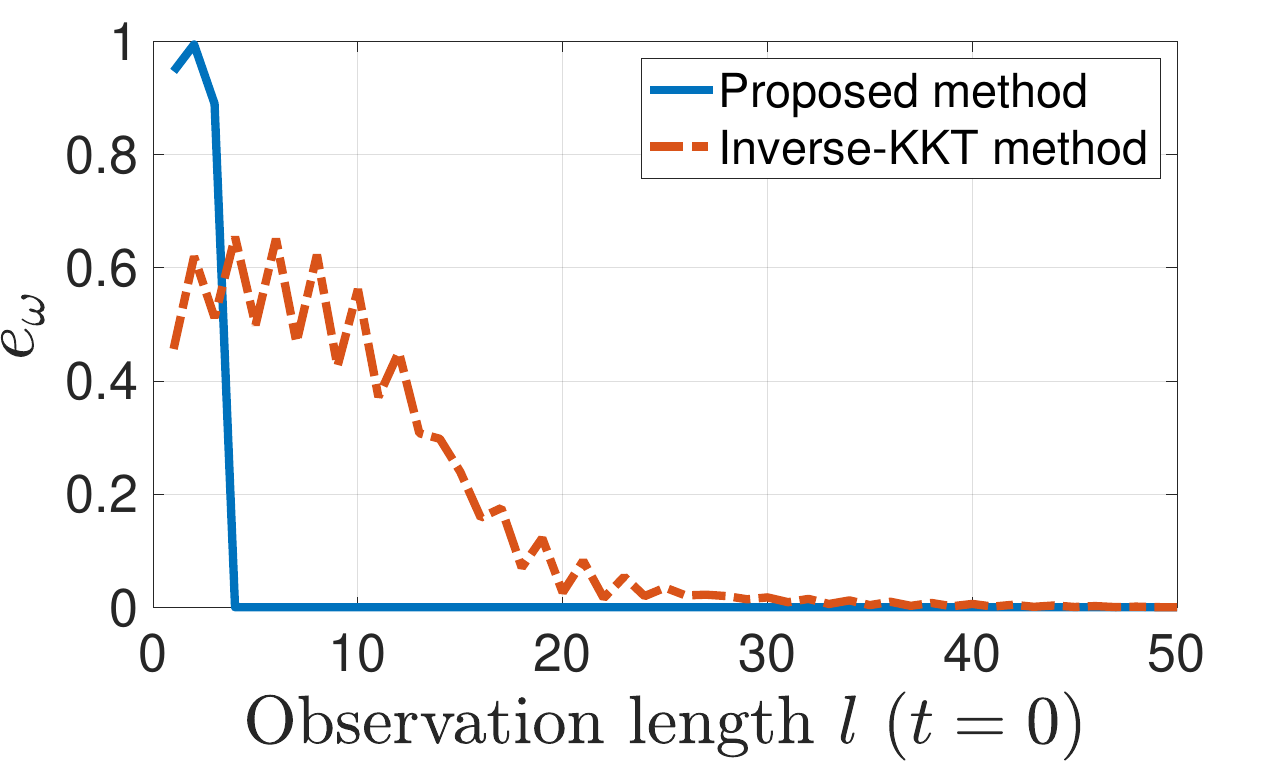}
		\caption{Observation starting from $t=1$}
		\label{comparefig.1}
	\end{subfigure}
	\hfill
	\begin{subfigure}[b]{0.33\textwidth}
		\centering
		\includegraphics[width=\textwidth]{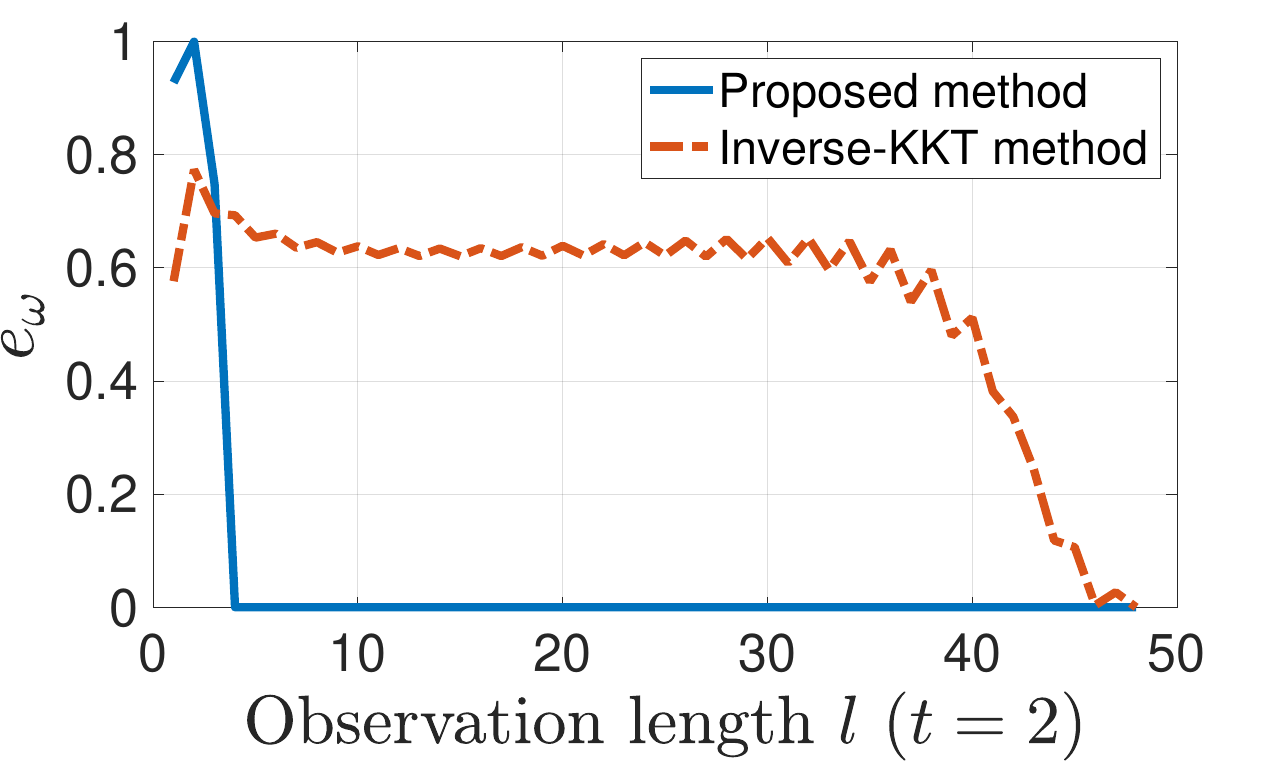}
		\caption{Observation starting from $t=2$}
		\label{comparefig.2}
	\end{subfigure}
	\hfill
	\begin{subfigure}[b]{0.33\textwidth}
		\centering
		\includegraphics[width=\textwidth]{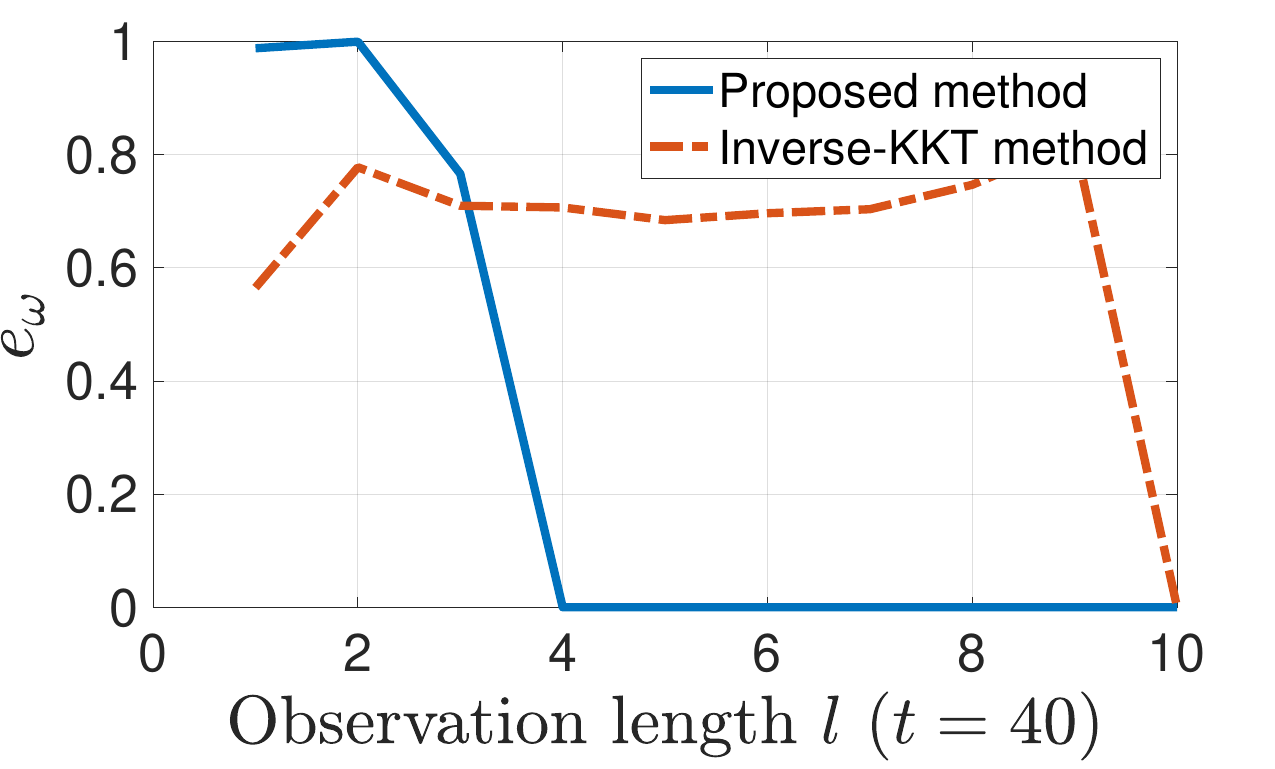}
		\caption{Observation starting from $t=40$}
		\label{comparefig.3}
	\end{subfigure}
	\caption{Comparison between the inverse-KKT method  (\ref{invKKT-incomplete}) and proposed recovery matrix method (\ref{proposed-incomplete})  when given incomplete trajectory observation  $\boldsymbol{\xi}_{t:t+l}$. Different observation starting time $t$ is used: $t=0$ in (a), $t=2$ in (b), and $t=40$ in (c). For each case, we increase the observation length $l$ from $1$ to the end of the horizon, i.e., $t+l=T$, and for each $l$, the estimation error $e_{\boldsymbol{\omega}}$ for both methods is evaluated, respectively. Note that the estimation error is defined in (\ref{estimationerror}).}
	\label{comparefig}
\end{figure*}

Here we demonstrate  how the recovery matrix is able to solve IOC problems using incomplete observations. We will show this by comparing with a recent inverse-KKT method developed in \citep{englert2017inverse}. The idea of the inverse-KKT method is based on the optimality equations similar to  (\ref{compareequ1}) using  full trajectory data  $\boldsymbol{\xi}$. 
As suggested by \citep{englert2017inverse},  the  weights are estimated by minimizing
\begin{equation}\label{compare2 }
\min_{\boldsymbol{\omega}}\norm{\boldsymbol{M}(\boldsymbol{\xi})\boldsymbol{\omega}}^2,
\end{equation}
subject to $\sum_{i}{\omega}_i=1$. Although the inverse-KKT method  \citep{englert2017inverse} is developed based on full trajectory data $\boldsymbol{\xi}$, we here want to see its performance  when only  incomplete data  $\boldsymbol{\xi}_{t:t+l} \subseteq \boldsymbol{\xi}$ is given.

As analyzed in Section \ref{relationshiptoexistingwork}, the coefficient matrix $\boldsymbol{M}(\boldsymbol{\xi})$ is a special case  of the  recovery matrix when $\boldsymbol{\xi}_{t:t+l}=\boldsymbol{\xi}_{0:T}$, that is, 
\begin{align}\label{compare1}
\boldsymbol{M}(\boldsymbol{\xi})=\boldsymbol{H}_1(0,T).
\end{align}
 {Recall that this is because  the LQR  in (\ref{lineardyn})-(\ref{lqrcost}) is a free-end optimal control system,  as analyzed in Section \ref{definitionrecoverymatrix}, $\boldsymbol{\lambda}_{T+1}=\boldsymbol{0}$. }
Given incomplete observation data $\boldsymbol{\xi}_{t:t+l} \subseteq \boldsymbol{\xi}$, comparing the inverse-KKT method
 	\begin{equation}\label{invKKT-incomplete}
 \min_{\boldsymbol{\omega}}\norm{\boldsymbol{M}(\boldsymbol{\xi}_{t:t+l})\boldsymbol{\omega}}^2 \quad \text{s.t.}\quad \sum_{i}^{}{\omega}_i=1,
 \end{equation}
with the proposed recovery matrix method
\begin{align}\label{proposed-incomplete}
\min_{\boldsymbol{\omega},\boldsymbol{\lambda}}\norm{{\boldsymbol{H}}_1(t,l)\boldsymbol\omega +{\boldsymbol{H}}_2(t,l)\boldsymbol\lambda
}^2  \quad \text{s.t.} \quad \sum_{i}{\omega}_i=1,
\end{align}
can show us how the \emph{unseen future data}  influences learning of the cost function.

For the LQR trajectory in Fig. \ref{figlqr}, we use the feature set ${\mathcal{F}}=\{x_{1}^2, x_{2}^2, u^2\}$, set the observation starting time $t$ to be 0, 2, and 40, respectively, and for each observation starting time $t$, we increase the observation length $l$ from $1$ to the end of the trajectory, i.e., $t+l=T$. With each observation $\boldsymbol{\xi}_{t:t+l}$,  we solve the weight estimate  using  the inverse-KKT method (\ref{invKKT-incomplete})  and the proposed method (\ref{proposed-incomplete}) and evaluate the estimation error $e_{\boldsymbol{\omega}}$ in (\ref{estimationerror}), respectively.
Results  are shown in Fig. \ref{comparefig}, based on which we have the following comments.

1) The inverse-KKT method is sensitive to the starting time of the observation sequence. When the observation starts from  $t=0$ (Fig. \ref{comparefig.1}), the inverse-KKT method achieves a successful estimate after  observation length $l\geq 30$; when $t=2$ (Fig. \ref{comparefig.2})  and   $t=40$ (Fig. \ref{comparefig.3})  only when  $l$  reaches  the end of  trajectory,  can the inverse-KKT method  obtain the successful estimate.

2) As we have analyzed in Section \ref{relationshiptoexistingwork}, the success of the inverse-KKT method requires that the given data $\boldsymbol{\xi}_{t:t+l}$ itself minimizes the cost function, which is only guaranteed when the observation reaches the   trajectory end, i.e., $t+l=T$. This explains the results in Fig. \ref{comparefig.2} and \ref{comparefig.3}. Given incomplete  $\boldsymbol{\xi}_{t:t+l}$ ($t+l<T$), although the inverse-KKT method still achieves a successful estimate  in Fig. \ref{comparefig.1}, such performance is not guaranteed and  heavily relies on   `informativeness' of  the given incomplete data relative to \emph{unseen future information}. In Fig.~\ref{figlqr}, since the trajectory data at  beginning phase  is more `informative' than the rest,  the inverse-KKT method starting from $t=0$ uses less data to converge (Fig. \ref{comparefig.1}) than  starting from $t=2$ (Fig. \ref{comparefig.2}).

3) In contrast, Fig. \ref{comparefig} shows  the effectiveness of using the recovery matrix to deal with incomplete observations. The proposed method  guarantees a successful estimate after  a much smaller observation length (e.g., around $l=4$ for all three cases). This advantage is because the  unseen future information is accounted for by $\boldsymbol{H}_2(t,l)$ in the recovery matrix and the related unknown future variable $\boldsymbol{\lambda}_{t+l+1}$ is jointly estimated in  (\ref{proposed-incomplete}).

In sum, we  make the following conclusions. First,  existing KKT-based methods generally require a full  trajectory, and cannot deal with   incomplete trajectory  data. Second, the proposed recovery matrix method addresses this by  jointly accounting for  \emph{unseen} future information; and 
the recovery matrix presents a systematic way to check whether a trajectory segment is sufficient to recover the objective function and if so, to solve it only using the  segment data. Third, existing KKT-based methods can be viewed as a special case of the proposed recovery matrix method when the segment data is the full trajectory.

\subsubsection{ {IOC for Infinite-horizon LQR.}} \label{IOCinfLQR}

 {We demonstrate the ability of the proposed method to solve the IOC problem for an infinite-horizon control system. We still use the LQR system in (\ref{lineardyn})-(\ref{lqrcost}) as an example, but here we set the time horizon $T=\infty$ (other conditions and parameters remain the same). The  optimal trajectory in this case is a result of feedback control $\boldsymbol{u}=K\boldsymbol{x}$ with a constant  control gain $K$ solved by  the algebraic Riccati equation \citep{bertsekas1995dynamic}. For the above infinite-horizon LQR  (with $Q$ and $R$ in (\ref{lqrcosts})), the  control gain is solved as  $K=[-0.1472\,\,  -0.1918]$. 
}

\begin{figure}[h]
	\centering
	\includegraphics[width=0.85\columnwidth]{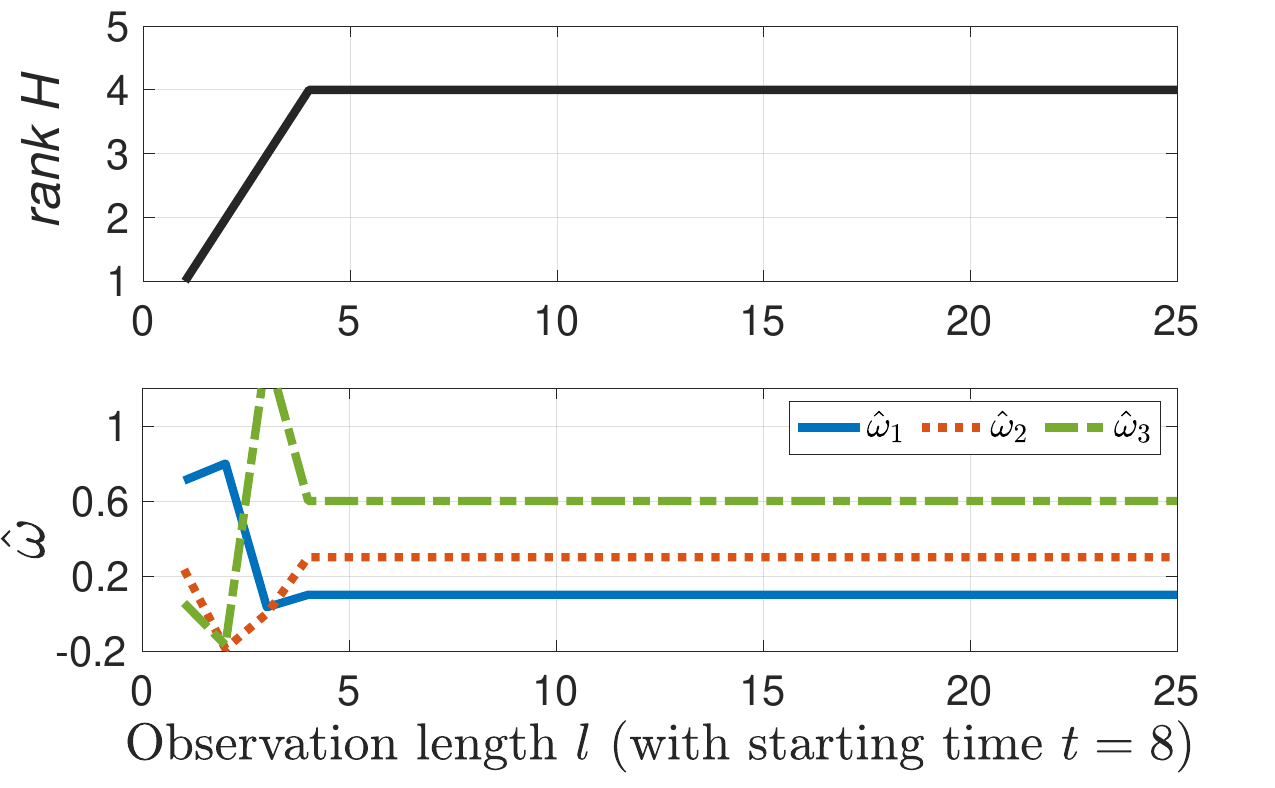}
	\caption{ {IOC results for infinite-horizon LQR system. The observation starting time is $t=8$ and the observation length $l$ increases from $l=1$ to $25$. The upper panel shows the rank of the recovery matrix versus increasing  $l$, and the bottom panel is the corresponding  weight estimate for each $l$.}}
	\label{infinite_LQR}
\end{figure}

 {In IOC, suppose that we observe an arbitrary segment from the infinite-horizon trajectory; here we use the segment data within the  time interval $[t,t+l]=[8,33]$, namely, $\boldsymbol{\xi}_{8:33}$ with $t=8$ and $l=25$. We set the candidate feature set ${\mathcal{F}}=\{x_{1}^2, x_{2}^2, u^2\}$. The IOC results using  Algorithm \ref{algorithm1} are presented in Fig. \ref{infinite_LQR}. Here we fix the observation starting time $t=8$ while increasing $l$ from 1 to the time interval end $25$. The upper panel of Fig. \ref{infinite_LQR} shows $\rank\boldsymbol{H}(8,l)$ versus  increasing observation length $l$, and the bottom panel shows the  weight estimate $\hat{\boldsymbol{\omega}}$ of each $l$ solved from the kernel of the recovery matrix $\boldsymbol{H}(8,l)$. As shown in the upper panel, with the observation length  $l$ increasing, $\rank\boldsymbol{H}(8,l)$ quickly reaches the upper bound rank $r+n-1=4$ after $l\geq 4$, indicting the successful estimate of the weights as shown in the bottom panel. The results demonstrate the ability of the proposed method to solve IOC problems for infinite-horizon optimal control systems.}

{\subsection{Evaluation on a Two-link Robot Arm}}

\begin{figure}[h]
	\centering
	\includegraphics[width=0.8\columnwidth]{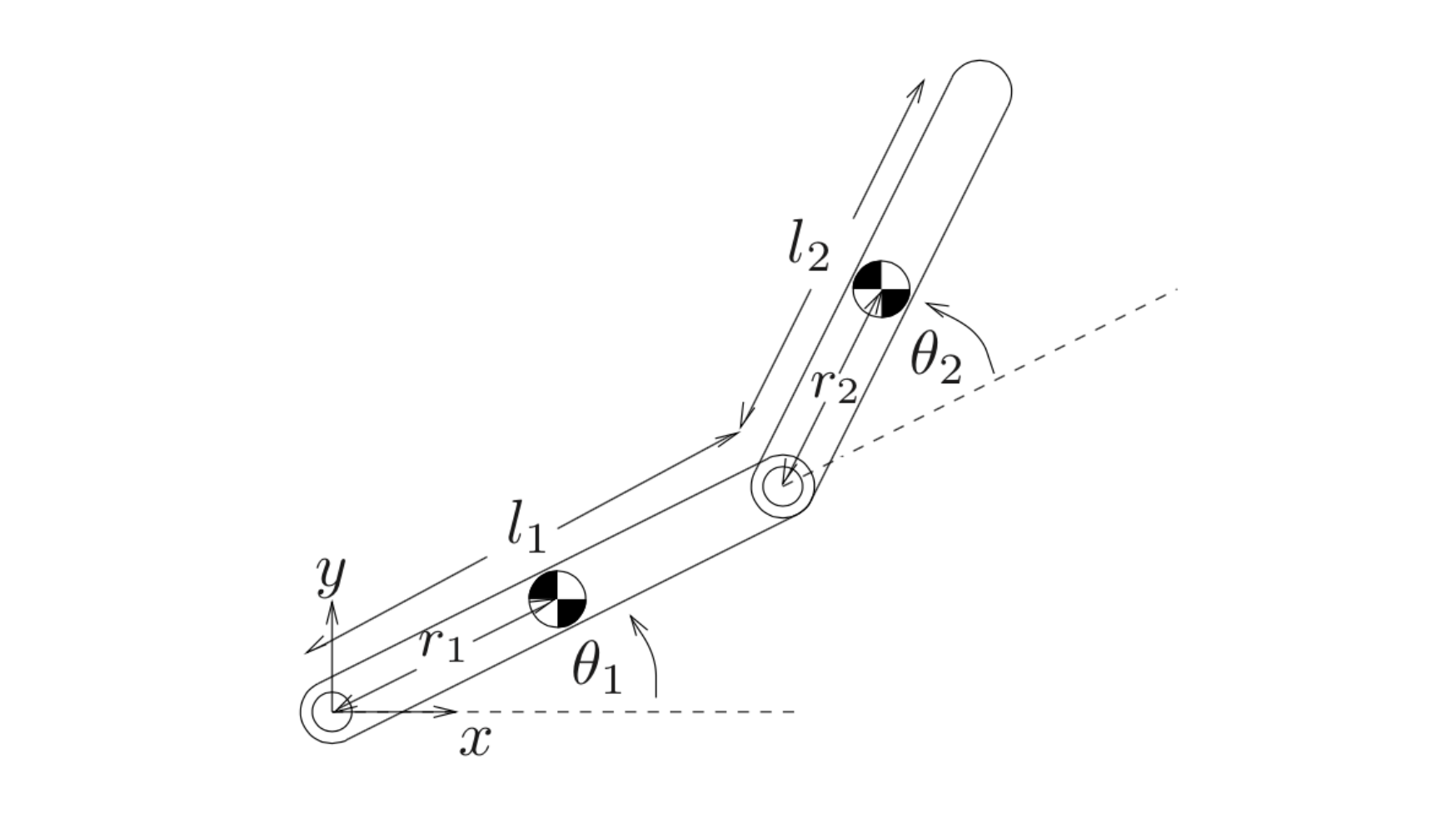}
	\caption{Two-link robot arm with coordinate definitions}
	\label{robotarm}
\end{figure}
To evaluate the proposed method on a non-linear plant, we use a two-link robot arm system, as shown in Fig. \ref{robotarm}.  {The  dynamics of the two-link arm} \cite[p. 209]{spong2008robot}  moving in the vertical plane  is
\begin{equation} \label{arminversedynamics}
M(\boldsymbol\theta)\ddot{\boldsymbol\theta}+C(\boldsymbol\theta,\dot{\boldsymbol\theta})\dot{\boldsymbol\theta}+\boldsymbol g(\boldsymbol\theta)=\boldsymbol\tau,
\end{equation}
where $\boldsymbol\theta=[\theta_1, \theta_2]^{\prime} \in \mathbb{R}^2$ is the joint angle vector; $M(\boldsymbol\theta)\in \mathbb{R}^{2\times2}$ is the  positive-definite inertia matrix; $C(\boldsymbol\theta,\dot{\boldsymbol\theta}) \in \mathbb{R}^{2\times2}$ is the Coriolis
matrix; $\boldsymbol g(\boldsymbol\theta) \in \mathbb{R}^2$ is the gravity vector; and $\boldsymbol\tau=[\tau_1, \tau_2]^{\prime}\in \mathbb{R}^2$ are the input torques applied to each joint.
The parameters  of the two-link robot arm in Fig. \ref{robotarm}  are as follows.
The mass of each link is $m_1=1 \mathrm{kg}$, $m_2=1 \mathrm{kg}$; the length of each link is $l_1=1 \mathrm{m}$, $l_2=1 \mathrm{m}$; the distance from the joint to the center of mass for each link is $r_1=0.5 \mathrm{m}$, $r_2=0.5 \mathrm{m}$; and  the moment of inertia with respect to the  center of mass for each link is $I_1=0.5 \mathrm{kgm^2}$, $I_2=0.5 \mathrm{kgm^2}$. From (\ref{arminversedynamics}), we have
\begin{equation} \label{armforwarddynamics}
\ddot{\boldsymbol\theta}=M(\boldsymbol\theta)^{-1}(-C(\boldsymbol\theta,\dot{\boldsymbol\theta})\dot{\theta}-\boldsymbol g(\boldsymbol\theta)+\boldsymbol\tau),
\end{equation}
which can be further expressed in state-space representation
\begin{equation}\label{armcontinus}
\dot{\boldsymbol{x}}=\boldsymbol f(\boldsymbol{x},\boldsymbol{u}),
\end{equation}
with the system state and  input defined as
\begin{equation}
\boldsymbol{x}=\begin{bmatrix}
\theta_1&\theta_2&\dot{\theta_1}&\dot{\theta_2}
\end{bmatrix}^{\prime},\quad  \quad \boldsymbol{u}=
\begin{bmatrix}
\tau_1&\tau_2
\end{bmatrix}^{\prime},
\end{equation}
respectively. 
  {We consider the following finite-horizon fixed-end optimal control  for the above robot arm system: }
\begin{equation}\label{armdirectproblem}
\begin{aligned}
& \underset{\boldsymbol{{x}}_{1:T}}{\text{min}}
& & \sum_{k=0}^{T} \boldsymbol \omega\boldsymbol \phi^*(\boldsymbol{x}_k,\boldsymbol{u}_k), \\
& \text{s.t.}
& & \boldsymbol x_{k+1}=\boldsymbol x_k+\Delta  \boldsymbol f(\boldsymbol x_k,\boldsymbol u_{k}),\\
&&& \boldsymbol x_0=\boldsymbol x_{\text{start}}, \\
&&& \boldsymbol x_{T+1}=\boldsymbol x_{\text{goal}},
\end{aligned}
\end{equation}
where $\Delta =0.01 \mathrm{s}$ is the discretization interval.
In (\ref{armdirectproblem}), we specify the initial state $\boldsymbol x_{\text{start}}=[0,0,0,0]^{\prime}$, goal state $\boldsymbol x_{\text{goal}}=[\frac{\pi}{2},-\frac{\pi}{2},0,0]^{\prime}$,  the time horizon $T=100$, and the feature vector and the  corresponding weights  
 {\begin{equation} \label{armfeaturesCost}
	\boldsymbol\phi^*=\begin{bmatrix}
	\tau_1^2 &\tau_2^2 & \tau_1\tau_2
	\end{bmatrix}^{\prime} \quad 
	\boldsymbol\omega=\begin{bmatrix}
	0.6 & 0.3 & 0.1 
	\end{bmatrix}^\prime,
	\end{equation}}%
respectively.
We solve the above optimal control system (\ref{armdirectproblem}) using the CasADi software \citep{Andersson2019} and plot the resulting trajectory  in Fig. \ref{generateplot}.

\begin{figure}[h]
	\centering
 	\includegraphics[width=0.9\columnwidth]{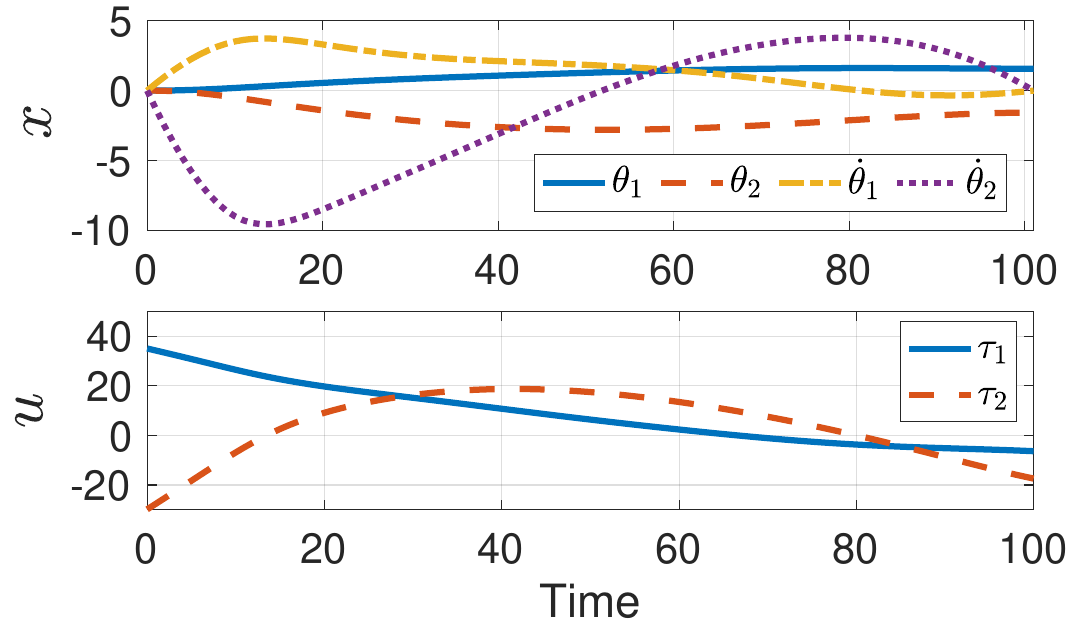}
	\caption{The optimal trajectory of the two-link  robot arm optimal control system (\ref{armdirectproblem}) with the cost function  (\ref{armfeaturesCost}).}
	\label{generateplot}
\end{figure}

\bigskip
\subsubsection{Minimal Required Observations for IOC.\label{rankpropertysimulation2}} Based on the above robot arm system, we first  show the use of the recovery matrix to check whether an incomplete trajectory observation suffices for  the minimal  observation required for  successful IOC. As an example, in Fig. \ref{generateplot}, we set the observation starting time at $t=50$. While increasing the observation length $l$ from $1$, we check the rank of  $\boldsymbol{H}(50,l)$, solve the weight estimate $\hat{\boldsymbol{\omega}}$ from the kernel of $\boldsymbol{H}(50,l)$, and evaluate the estimation error $e_{\boldsymbol{\omega}}$ in (\ref{estimationerror}) for $\hat{\boldsymbol{\omega}}$. This process is repeated for  three different candidate feature sets:  {$\mathcal{F}=\{\tau_1^2,\tau_2^2, \tau_1\tau_2\}, \mathcal{F}=\{\tau_1^2,\tau_2^2, \tau_1\tau_2, \tau_1^3,\tau_2^3,  \tau_1^2\tau_2\},$ and $ \mathcal{F}=\{\tau_1^2,\tau_2^2, \tau_1\tau_2, \tau_1^3,\tau_2^3,  \tau_1\tau_2^2, \tau_1^4,\tau_2^4, \tau_1^3\tau_2, \tau_1^2\tau_2^2  \}$}, respectively. The results are shown in Fig.~\ref{robotrank}.

\begin{figure}[h]
	\includegraphics[width=0.9\columnwidth]{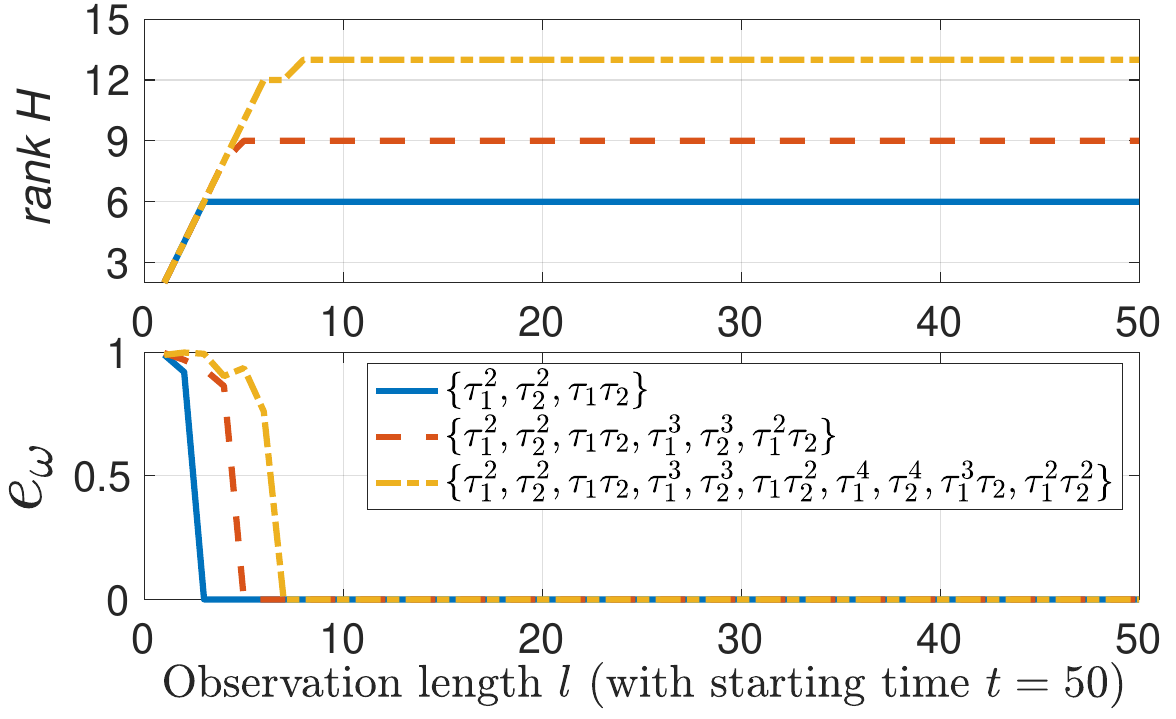}
	\caption{The rank of the recovery matrix and corresponding estimation error $e_{\boldsymbol{\omega}}$ under different observation length $l$ and different candidate feature sets. The upper panel shows the rank of the recovery matrix versus $l$, and the bottom panel shows the  estimation error $e_{\boldsymbol{\omega}}$ for each $l$. Three   feature sets,  {$\mathcal{F}=\{\tau_1^2,\tau_2^2, \tau_1\tau_2\}, \mathcal{F}=\{\tau_1^2,\tau_2^2, \tau_1\tau_2, \tau_1^3,\tau_2^3,  \tau_1^2\tau_2\},$ and $ \mathcal{F}=\{\tau_1^2,\tau_2^2, \tau_1\tau_2, \tau_1^3,\tau_2^3,  \tau_1\tau_2^2, \tau_1^4,\tau_2^4, \tau_1^3\tau_2, \tau_1^2\tau_2^2  \}$}, are used, respectively, and the corresponding results are plotted in different  lines.  {Note that when $\rank\boldsymbol{H}(50,l)<|\mathcal{F}|+n-1$, since the dimension of the kernel of $\boldsymbol{H}(50,l)$ is at least 2,  we thus choose  $\hat{\boldsymbol\omega}$ from the kernel of $\boldsymbol{H}(50,l)$ randomly}.}
	\label{robotrank}
\end{figure}

Results in the upper panel of Fig. \ref{robotrank} show that additional observations   increase the rank of the recovery matrix. Once the additional observations lead to the upper-bound rank of the recovery matrix, i.e., $|\mathcal{F}|+n-1$, the corresponding length is  the minimal observation length $l_{\min}(t)$ required for a successful weight estimate, as  shown in the corresponding bottom panel in Fig. \ref{robotrank}. Moreover,  including additional irrelevant features in $\mathcal{F}$ will lead to the increased  minimal required observation length $l_{\min}(t)$, as implied  by (\ref{unifomlb}). This will be  discussed later in Section \ref{irrelevantfeatures}.

\subsubsection{Observation Noise.\label{observationnoise}}  We   test the proposed incremental IOC approach (Algorithm \ref{algorithm1}) under different data noise levels. We add  to the trajectory  in Fig.~\ref{generateplot}  Gaussian noise of different levels that are characterized by different standard deviations from $\sigma=10^{-5}$ to $\sigma=10^{-1}$. In Algorithm \ref{algorithm1}, we use  {$\mathcal{F}=\{{\tau}^2_1,{\tau}^2_2, \tau_1\tau_2\}$} and set $\gamma=45$ (the choice of $\gamma$ will be discussed later in Section \ref{parametersetting}).

We set  the observation starting time $t$  at all  time instants except for those near the trajectory end which can not provide sufficient subsequent observation length. As an example, we present the experimental results for the case of noise level $\sigma=10^{-2}$ in Fig \ref{robotnoise}. Here, the upper panel shows  the minimal required observation length $l_{\min}(t)$ automatically found for each observation starting time $t$, and  the bottom shows the corresponding weight estimate using the minimal required observation data $\boldsymbol{\xi}_{t:t+l_{\min}(t)}$.

\begin{figure}[h]
	\centering
	\includegraphics[width=0.85\columnwidth]{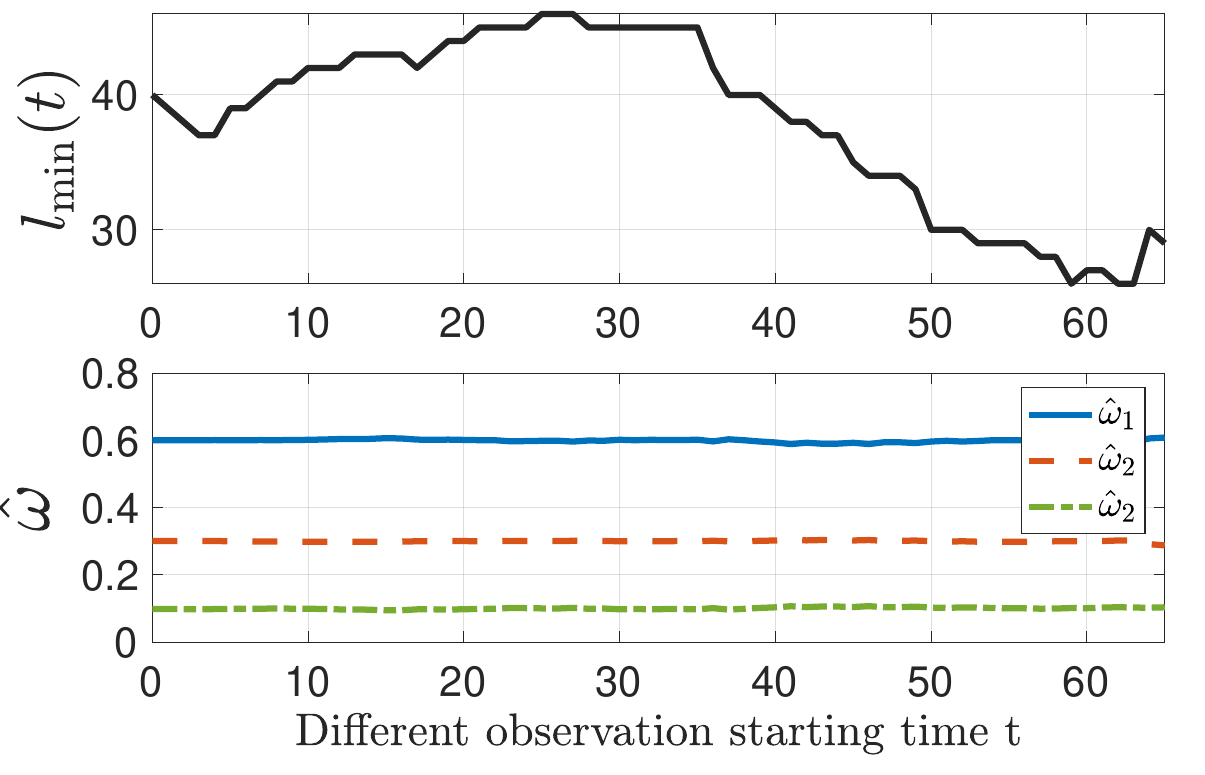}
	\caption{IOC by automatically finding the minimal required observation  under noise level $\sigma=10^{-2}.$ The x-axis is the different observation starting time $t$. The upper panel shows the automatically-found minimal  required observation length $l_{\min}(t)$ at different $t$, and the bottom panel shows the corresponding estimate $\boldsymbol{\hat\omega}$  via (\ref{compute_aug_cost}). Note that ground truth $\boldsymbol{\omega}=[0.6,0.3, 0.1]^\prime$.}
	\label{robotnoise}
\end{figure}

From Fig. \ref{robotnoise}, we   see that the automatically-found minimal required observation length $l_{\min}(t)$ varies depending on the observation starting time $t$. This can be interpreted by noting that the trajectory  data in Fig. \ref{generateplot} in different intervals  has different informativeness to reflect the cost function. For example, according to Fig.~\ref{robotnoise}, we can postulate that the beginning and final portions of the  trajectory data are  more `data-informative' than other portions,  thus needing smaller  $l_{\min}(t)$  to achieve the successful estimate. This can be understood if we consider  that  the beginning and final portions of the trajectory  in Fig. \ref{generateplot} has richer  patterns such as curvatures   than the middle which are more smooth. Using the recovery matrix, the data informativeness about the cost function is quantitatively indicated by the recovery matrix's rank.	Even under  observation noise, the proposed method can  adaptively find the sufficient observation length size such that the data is informative enough to guarantee  a successful estimate of the  weights, as shown by both upper and bottom panels in Fig. \ref{robotnoise}.

\begin{table}[h]
	\small\sf\centering
	\begin{threeparttable}
		\caption{Results of incremental IOC (Algorithm \ref{algorithm1}, $\gamma=45$) under different noise levels.\label{talbe1}}
		\begin{tabular}{lll}
			\toprule
			Noise level $\sigma$ & Averaged $l_{\min}/T$ (\%) \tnote{\textdagger} &  Averaged  $e_{\boldsymbol{\omega}}$\tnote{\textdagger}  \\
			\midrule
			
			$\sigma=10^{-5}$   &$8\% $  &   $4.3\times 10^{-4}$  \\
			$\sigma=10^{-4}$   & $8.1\%$ &   $4.0\times10^{-3}$ \\
			$\sigma=10^{-3}$  &$12.61\%$ &    $8.1\times10^{-3}$ \\
			$\sigma=10^{-2}$  &$33.8\%$&   $8.5\times10^{-3}$ \\
			$\sigma=10^{-1}$  &$70.0$\%&   $7.1\times10^{-3}$ \\
			\bottomrule
		\end{tabular}
		\begin{tablenotes}
			\small
			\item[\textdagger] The average is calculated based on all successful estimations over all  observation cases (varying observation starting time). 
		\end{tablenotes}
	\end{threeparttable}
\end{table}

We summarize all results under different  noise levels  in Table \ref{talbe1}. Here the minimal required observation length is presented in percentage  with respect to the total  horizon $T$.
According to Table \ref{talbe1}, under a fixed rank index threshold  (here $\gamma=45$), we can see that high noise levels, on average, will lead to larger  minimal required observation length, but the estimation error is not influenced too much. This is because the increased observation length can compensate for the uncertainty induced by data noise and finally produces a `neutralized' estimate. Hence, the results  prove the  robustness of the proposed incremental IOC algorithm against the small observation  noise. We will later show how to further improve the accuracy by adjusting $\gamma$.

\subsubsection{Presence of Irrelevant Features.\label{irrelevantfeatures}} We here assume that exact knowledge of relevant features is not available, and we evaluate the  performance of Algorithm \ref{algorithm1} given a feature set including irrelevant features. We add all observation data with   Gaussian noise of $\sigma=10^{-3}$. 
In Algorithm \ref{algorithm1}, we set $\gamma=45$ and construct a  feature set $\mathcal{F}$ based on the following     candidate features
 {\begin{multline} \label{candidatefeatures}
	\{\tau_1^2,\,\tau_2^2, \,\tau_1\tau_2,\,\,\, \tau_1^3,\,\tau_2^3, \, \tau_1\tau_2^2,\, \tau_1^2\tau_2, \,\\
	\tau_1^4,\tau_2^4,\,\, \tau_1^3\tau_2, \,\,\tau_1\tau_2^3,\, \,\tau_1^2\tau_2^2  \}.
	\end{multline}}\noindent
Algorithm \ref{algorithm1} is applied the same way as in the previous experiment: by starting the observation at all time instants except for those near the trajectory end.  { We provide different candidate  feature sets in the first column in Table~\ref{talbe2}},  and for each case we compute the average of the minimal required observation length and the average of estimation error in (\ref{estimationerror}). The results are  summarized in second and third columns in Table \ref{talbe2}.

\begin{table}[h]
	\small\sf\centering
	\renewcommand{\arraystretch}{1.3}
	\begin{threeparttable}
		\caption{Results of incremental IOC (Algorithm \ref{algorithm1}, $\gamma=45$) with different given  feature sets.\label{talbe2}}
		\begin{tabular}{lll}
			\toprule
			Candidate feature set $\mathcal{F}$ & Averaged $l_{\min}/T$\tnote{1}  &  Averaged $e_{\boldsymbol{\omega}}$\tnote{1}  \\
			\midrule
			$\{\tau_1^2,\,\tau_2^2, \,\tau_1\tau_2\}$    &$12.18\%$&   $4.2\times 10 ^{-3}$  \\[0.8ex] 
			$\begin{aligned}
			\{\tau_1^2,\,\tau_2^2, \,\tau_1\tau_2, \, \tau_1^3, \tau_2^3\}
			\end{aligned}$   & $14.7\%$ &   $9.7\times 10 ^{-3}$  \\[0.8ex] 
			$\begin{aligned}
			\{&\tau_1^2,\,\tau_2^2, \,\tau_1\tau_2, \, \tau_1^3, \tau_2^3, \\[-3pt] & \tau_1\tau_2^2,\, \tau_1^2\tau_2\}
			\end{aligned}$  &$25.69\%$ &  $8.7\times 10 ^{-3}$ \\[2.2ex] 
			$\begin{aligned}
			\{&\tau_1^2,\,\tau_2^2, \,\tau_1\tau_2, \,\tau_1^3, \tau_2^3,  \\[-3pt] & \tau_1\tau_2^2,\, \tau_1^2\tau_2, \tau_1^4,\tau_2^4\}
			\end{aligned}$  &$35.97\%$ &   $8.6\times 10 ^{-3}$\\[2.2ex] 
			$\begin{aligned}
			\{&\tau_1^2,\,\tau_2^2, \,\tau_1\tau_2, \,\tau_1^3, \tau_2^3,\\[-3pt] &   \tau_1\tau_2^2,\, \tau_1^2\tau_2,  \tau_1^4,\tau_2^4, \\[-3pt] &  \tau_1^3\tau_2, \,\tau_1\tau_2^3, \,\tau_1^2\tau_2^2 \}
			\end{aligned}$  &$45.53\%$&   $9.1\times 10 ^{-3}$ \\
			\bottomrule
		\end{tabular}
		\begin{tablenotes}
			\small
			\item[1] The average is calculated based on all successful estimations over all  observation cases (varying observation starting time). 
		\end{tablenotes}
	\end{threeparttable}
\end{table}

Table \ref{talbe2} indicates that on average,  the minimal required observation length increases as additional irrelevant features are included  to the feature set $\mathcal{F}$. This can be understood if we consider (\ref{minobservations}) and the rank non-decreasing property in Lemma \ref{lemma2}:  when a certain number of irrelevant features are added, the rank required for successful estimate will increase by the same amount,  thus needing additional trajectory data points. Due to  increased observation length,  the estimation  accuracy is not much   influenced by the  additional irrelevant features. Thus we conclude that the proposed incremental IOC algorithm applies to the presence of irrelevant features.

\subsubsection{Parameter Setting.} \label{parametersetting}
We now discuss how to choose the rank  threshold $\gamma$ in Algorithm \ref{algorithm1}. Since in Algorithm \ref{algorithm1} the rank index  (\ref{rankindex}) for the recovery matrix is used to find the minimal required observation length, we first investigate how 
the rank index $\kappa(t,l)$  changes as the observation length $l$ increases. We  use the trajectory data in Fig. \ref{generateplot}  with added Gaussian noise of $\sigma=10^{-3}$, $\sigma=2\times10^{-3}$, and $\sigma=10^{-2}$, respectively. The candidate features set here is  {$\mathcal{F}=\{\tau_1^2,\tau_2^2, \tau_1\tau_2\}$}. We fix the observation start time $t=0$ and increase the observation length $l$ from $1$ to $T$. The rank index $\kappa(t=0,l)$ for different $l$ is shown in Fig. \ref{robotkappa}.

\begin{figure}[h]
	\includegraphics[width=0.9\columnwidth]{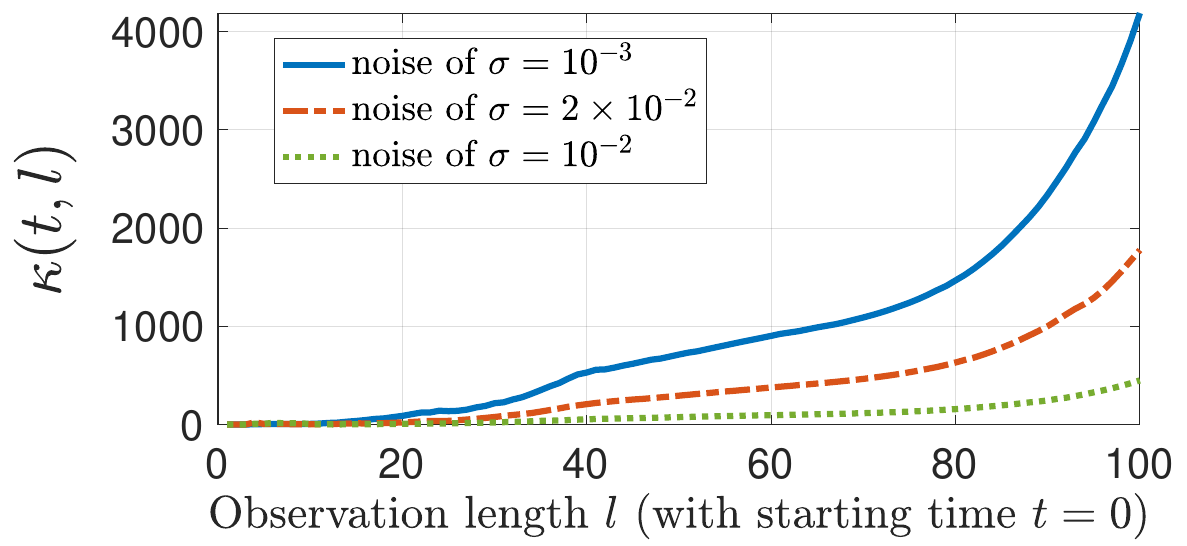}
	\caption{The rank index $\kappa(0,l)$ in (\ref{rankindex}) versus different observation length $l$  under different  noise levels.}
	\label{robotkappa}
\end{figure}

From Fig. \ref{robotkappa}, we can see that although  $\kappa(t,l)$ has different scales at different  noise levels, it in general increases as the observation length $l$ increases. This can be understood  if we compare the above results to $\kappa(t,l)$ in  noise-free cases: when there is no data noise,  according to Lemma \ref{lemma2} and \ref{lemma3}, as $l$ increases, $\kappa(t,l)$ will first remain  zero when $l<l_{\min}$, then  increase to infinity after $l\geq l_{\min}$. In noisy settings, $\kappa(t,l)$ however will increase to a large finite value. 
From the plot, we can postulate that in practice choosing a larger threshold $\gamma$ will lead to a larger minimal required  observation length $l_{\min}$, thus more  data points will be included into the recovery matrix to compute the estimate of the weights, which may finally improve the estimation accuracy (similar to results in Table~\ref{talbe1}). In what follows, we will verify this postulation by showing how $\gamma$ affects the performance of Algorithm \ref{algorithm1}.

\begin{figure}[h]
	\centering
	\includegraphics[width=0.88\columnwidth]{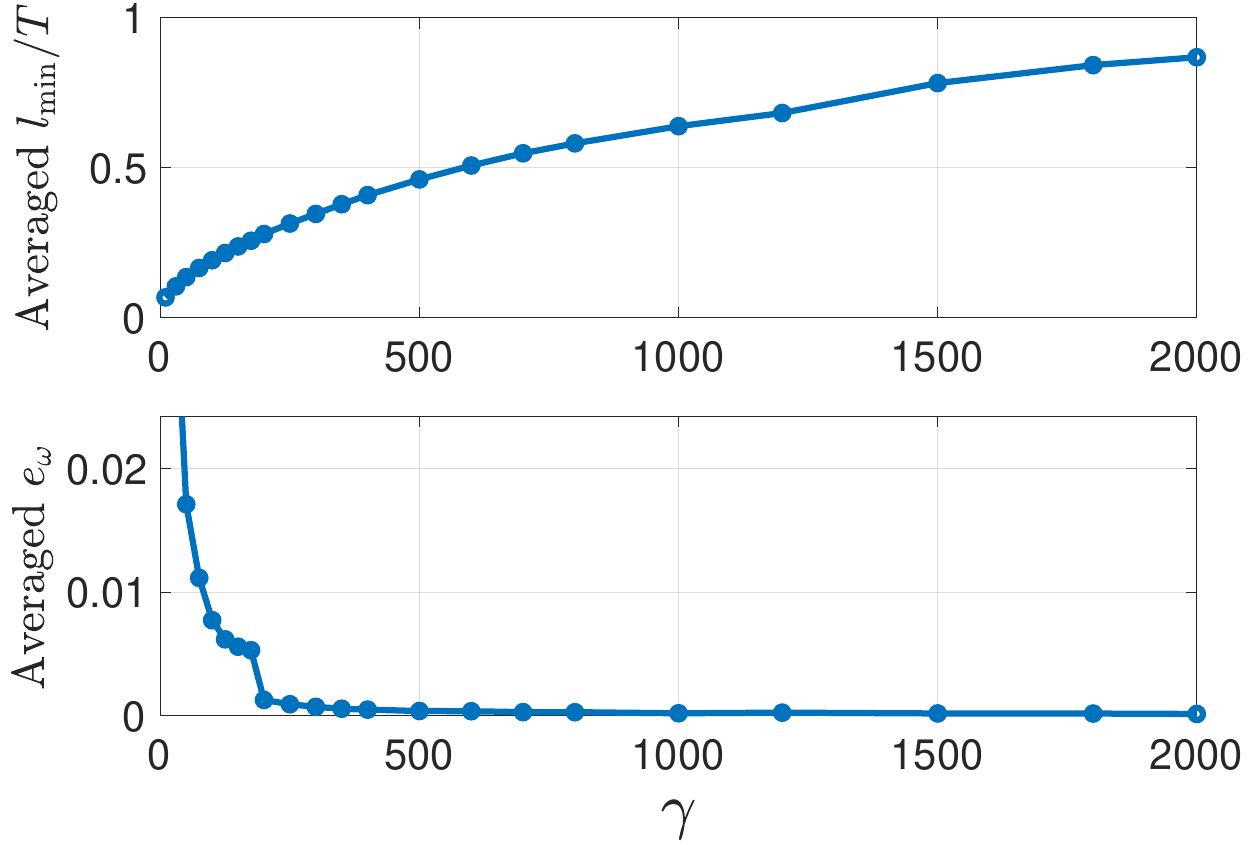}
	\caption{Averaged $l_{\min}$ (upper panel) and averaged estimation error $e_{\boldsymbol{\omega}}$ (bottom panel) for different choices of $\gamma$.}
	\label{parameters}
\end{figure}

We add Gaussian noise $\sigma=10^{-3}$ to  the trajectory data in Fig. \ref{generateplot}, and apply Algorithm \ref{algorithm1} by starting  the observation at all possible time steps, as performed in  previous experiments. We vary  $\gamma$ to show its influence on the average of the minimal required observation length $l_{\min}$ and the average of the estimation error $e_{\boldsymbol{\omega}}$. The results are shown in Fig.~\ref{parameters}, from which
we  can observe that first, a larger $\gamma$ will lead to  larger minimal required observation length; and second, due to the increased minimal required  observation, the corresponding estimation accuracy is improved because data noise or other error sources can be compensated by  additional observation data. These facts thus prove our previous postulation based on Fig. \ref{robotkappa}. Moreover, Fig. \ref{parameters} also shows  as $\gamma$ exceeds a certain value, e.g., 200, continuously increasing $\gamma$ will not improve the recovery accuracy significantly. This suggests that the choice of $\gamma$ is not sensitive to the performance if  $\gamma$  is large. Therefore, in practice it is possible to find a proper $\gamma$ without much manual effort such that both the estimation accuracy and computational cost are balanced.

\section{Conclusions}
This article considers the problem of learning an objective function  from an observation of an incomplete  trajectory.
To achieve this goal, we develop the recovery matrix, which establishes a relationship between  trajectory segment data and the unknown   weights of given candidate features. The rank of the recovery matrix indicates whether an incomplete trajectory observation is sufficient  for obtaining a successful estimate of the weights.
By investigating the properties of the recovery matrix, we further demonstrate that  additional observations may increase the rank of the recovery matrix, thus contributing to enabling the successful estimation, and that  the IOC  can be processed incrementally. Based on the recovery matrix, a method for using incomplete trajectory observations   to estimate the weights of specified features is established,   and an incremental IOC algorithm  is developed by automatically finding the minimal required observation.

\begin{acks}
	This research is partly supported by the ERC~Consolidator~Grant  \emph {Safe data-driven control for human-centric systems}  under grant agreement 864686 at Chair of Information-oriented Control, Technical University of Munich.
\end{acks}

\bibliographystyle{SageH}
\bibliography{ijrrbib}

\appendix
\section{Proof of Lemma \ref{lemma1}} \label{prooflemma1}
Consider the recovery matrix $\boldsymbol{H}(t,l)$ for the trajectory segment  $\boldsymbol{\xi}_{ {t:t+l}}=(\boldsymbol{x}^*_{ {t:t+l}} \boldsymbol{u}^*_{ {t:t+l}})$, with $t$ being the observation starting time  and $l$ being the observation length. When a subsequent point   $\boldsymbol{\xi}_{t+l+1}=(\boldsymbol{x}^*_{ {t:t+l+1}}, \boldsymbol{u}^*_{ {t:t+l+1}})$ is observed, from Definition \ref{def_rm}, the updated recovery matrix is $\boldsymbol{H}(t,l+1)=[\boldsymbol{H}_1(t,l+1),\boldsymbol{H}_2(t,l+1)]$, where
\begin{multline} \label{H_new1}
\boldsymbol{H}_1(t,l+1)=
\boldsymbol{F}_u(t,l+1)\boldsymbol{F}^{-1}_x(t,l+1)\boldsymbol{\Phi}_x(t,l+1)\\+\boldsymbol{\Phi}_u(t,l+1),
\end{multline} and
\begin{equation}\label{H_new2}
\boldsymbol{H}_2(t,l+1)=
\boldsymbol{F}_u(t,l+1)\boldsymbol{F}_x^{-1}(t,l+1)\boldsymbol{V}(t,l+1).
\end{equation}
Here $\boldsymbol{F}_x(t,l+1)$, $\boldsymbol{F}_u(t,l+1)$, $\boldsymbol{\Phi}_x(t,l+1)$, $\boldsymbol{\Phi}_u(t,l+1)$, and $\boldsymbol{V}(t,l+1)$, defined in (\ref{HFx})-(\ref{Hv}), are updated as follow:
\begin{subequations}\label{updates}
	\begin{align}
	\boldsymbol{\Phi}_{{u}}(t,l+1) &=
	\begin{bmatrix}
	\boldsymbol{\Phi}_u(t,l) \\
	\frac{\partial \boldsymbol\phi^{\prime}}{\partial \boldsymbol{u}^*_{t+l}}
	\end{bmatrix}
	\label{phiupdate},\\
	\boldsymbol{\Phi}_x(t,l+1) &=
	\begin{bmatrix}
	\boldsymbol{\Phi}_x(t,l) \\
	\frac{\partial \boldsymbol\phi^{\prime}}{\partial \boldsymbol{x}^*_{ {t:t+l+1}}}
	\end{bmatrix}
	\label{phixpdate}, \\
	\boldsymbol{F}_u(t,l+1)&=
	\begin{bmatrix}
	\boldsymbol{F}_u(t,l)& \boldsymbol{0} \\
	\boldsymbol{0}& \frac{\partial \boldsymbol{f}^{\prime}}{\partial \boldsymbol{u}^*_{t+l}}
	\end{bmatrix}
	\label{fuupdate}, \\
	\boldsymbol{F}_x^{-1}(t,l+1) &=
	\begin{bmatrix}
	\boldsymbol{F}_x(t,l)& -\boldsymbol{V}(t,l) \\
	\boldsymbol{0}& \boldsymbol{I}
	\end{bmatrix}^{-1} \nonumber \\
	&=\begin{bmatrix}
	\boldsymbol{F}_x^{-1}(t,l) & \boldsymbol{F}_x^{-1}(t,l)\boldsymbol{V}(t,l)\\
	\boldsymbol{0} & \boldsymbol{I}
	\end{bmatrix},
	\label{fxupdate}
	\end{align}
\end{subequations}
respectively.
Here (\ref{fxupdate}) is  based on the fact
\begin{equation*}
\begin{bmatrix}
A&B\\
C&D
\end{bmatrix}^{-1}=\begin{bmatrix}
A^{-1}+A^{-1}BK^{-1}CA^{-1}&-A^{-1}BK^{-1}\\
-K^{-1}CA^{-1}&K^{-1}
\end{bmatrix}
\end{equation*} with $K=D-CA^{-1}B$ being the Schur complement of the above block matrix with respect to $A$.
Combining (\ref{phiupdate})-(\ref{fxupdate}), we have
\begin{align} \label{H_new1_temp}
\begin{split}
\boldsymbol{H}_1(t,l+1)&=\boldsymbol{F}_u(t,l+1)\boldsymbol{F}_x^{-1}(t,l+1)
\boldsymbol{\Phi}_x(t,l+1)\\&\qquad+\boldsymbol{\Phi}_u(t,l+1)
\end{split} \nonumber\\
\begin{split}
&=\begin{bmatrix}
\boldsymbol{F}_u(t,l)\boldsymbol{F}_x^{-1}(t,l)\boldsymbol{\Phi}_x(t,l)+\boldsymbol{\Phi}_u(t,l)\\
\frac{\partial \boldsymbol{f}^{\prime}}{\partial \boldsymbol{u}^*_{t+l}}\frac{\partial \boldsymbol\phi^{\prime}}{\partial \boldsymbol{x}^*_{ {t+l+1}}}+\frac{\partial \boldsymbol\phi^{\prime}}{\partial \boldsymbol{u}^*_{t+l}}
\end{bmatrix}\\
&+
\begin{bmatrix}
\boldsymbol{F}_u(t,l)\boldsymbol{F}_x^{-1}(t,l)\boldsymbol{V}(t,l)\frac{\partial \boldsymbol\phi^{\prime}}{\partial \boldsymbol{x}^*_{ {t+l+1}}} \\
\boldsymbol{0}
\end{bmatrix}.
\end{split}
\end{align}
Combining with (\ref{H1})-(\ref{H2}), the above (\ref{H_new1_temp})  becomes
\begin{align}\label{updateH1}
\boldsymbol{H}_{1}(t,l+1)&=
\begin{bmatrix}
\boldsymbol{H}_{1}(t,l)+\boldsymbol{H}_{2}(t,l)\frac{\partial \boldsymbol\phi^{\prime}}{\partial \boldsymbol{x}^*_{ {t+l+1}}} \\
\frac{\partial \boldsymbol{f}^{\prime}}{\partial \boldsymbol{u}^*_{t+l}}\frac{\partial \boldsymbol\phi^{\prime}}{\partial \boldsymbol{x}^*_{ {t+l+1}}}+\frac{\partial \boldsymbol\phi^{\prime}}{\partial \boldsymbol{u}^*_{t+l}}
\end{bmatrix}. 
\end{align}
Considering (\ref{phiupdate})-(\ref{fxupdate}),  we have
\begin{align}\label{updateH2}
\boldsymbol{H}_2(t,l+1)&=
\boldsymbol{F}_u(t,l+1)\boldsymbol{F}_x^{-1}(t,l+1)\boldsymbol{V}(t,l+1) \nonumber\\
&=\begin{bmatrix}
\boldsymbol{F}_u(t,l)\boldsymbol{F}_x^{-1}(t,l)\boldsymbol{V}(t,l)\frac{\partial \boldsymbol{f}^{\prime}}{\partial \boldsymbol{x}^*_{ {t+l+1}}} \\
\frac{\partial \boldsymbol{f}^{\prime}}{\partial \boldsymbol{u}^*_{t+l}}\frac{\partial \boldsymbol{f}^{\prime}}{\partial \boldsymbol{x}^*_{ {t+l+1}}}
\end{bmatrix} \nonumber\\
&=\begin{bmatrix}
\boldsymbol{H}_2(t,l)\frac{\partial \boldsymbol{f}^{\prime}}{\partial \boldsymbol{x}^*_{ {t+l+1}}} \\
\frac{\partial \boldsymbol{f}^{\prime}}{\partial \boldsymbol{u}^*_{t+l}}\frac{\partial \boldsymbol{f}^{\prime}}{\partial \boldsymbol{x}^*_{ {t+l+1}}}
\end{bmatrix}.
\end{align}
Finally joining (\ref{updateH1}) and (\ref{updateH2}) and writing them in the matrix form lead to (\ref{iterH1}).

When $l=1$, that is,   $\boldsymbol{\xi}_{t:t+1}=(\boldsymbol{x}^*_{t:t+1},\boldsymbol{u}^*_{t:t+1})$ is available, we have $\boldsymbol{F}_x(t,1)=I$, $\boldsymbol{F}_u(t,1)=\frac{\partial \boldsymbol{f}^{\prime}}{\partial \boldsymbol{u}^*_{t}}$, $\boldsymbol{\Phi}_x(t,1)=\frac{\partial \boldsymbol\phi^{\prime}}{\partial \boldsymbol{x}^*_{t+1}}$, $\boldsymbol{\Phi}_u(t,1)=\frac{\partial \boldsymbol\phi^{\prime}}{\partial \boldsymbol{u}^*_{t}}$, and  $\boldsymbol{V}(t,1)=\frac{\partial \boldsymbol{f}^{\prime}}{\partial \boldsymbol{x}^*_{t+1}}$. 
According to the definition of recovery matrix in (\ref{H})-(\ref{H2}), we thus obtain (\ref{iterH0}).  This completes the proof. $\qed$

\section{Proof of Lemma \ref{lemma2}} \label{prooflemma2}
From Lemma \ref{lemma1}, we have
\begin{align}
&\rank \boldsymbol{H}(t,l+1) \nonumber \\
&=\rank \begin{bmatrix}
\boldsymbol{H}_{1}(t,l) &\boldsymbol{H}_{2}(t,l) \\
\frac{\partial \boldsymbol\phi^{\prime}}{\partial \boldsymbol{u}^*_{t+l}}&
\frac{\partial \boldsymbol{f}^{\prime}}{\partial \boldsymbol{u}^*_{t+l}}
\end{bmatrix}
\begin{bmatrix}
I & \boldsymbol 0 \\
\frac{\partial \boldsymbol\phi^{\prime}}{\partial \boldsymbol{x}^*_{ {t+l+1}}}&
\frac{\partial \boldsymbol{f}^{\prime}}{\partial \boldsymbol{x}^*_{ {t+l+1}}}
\end{bmatrix}. \label{prooflemma2_equ1}
\end{align}
If $\det(\frac{\partial \boldsymbol{f}}{\partial \boldsymbol{x}^*_{ {t+l+1}}})\neq 0$, the last block matrix in (\ref{prooflemma2_equ1}) is non-singular. Consequently
\begin{align} \label{prooflemma2_equ2}
\rank \boldsymbol{H}(t,l+1)
&=\rank \begin{bmatrix}
\boldsymbol{H}_{1}(t,l) &\boldsymbol{H}_{2}(t,l) \\
\frac{\partial \boldsymbol\phi^{\prime}}{\partial \boldsymbol{u}^*_{t+l}}&
\frac{\partial \boldsymbol{f}^{\prime}}{\partial \boldsymbol{u}^*_{t+l}}
\end{bmatrix} \nonumber\\
&\geq \rank \begin{bmatrix}
\boldsymbol{H}_1(t,l) & \boldsymbol{H}_2(t,l)
\end{bmatrix} \nonumber\\
&= \rank \boldsymbol{H}(t,l). 
\end{align}
Note that both (\ref{prooflemma2_equ1}) and the inequality (\ref{prooflemma2_equ2}) are independent of the choice of $\boldsymbol{\phi}$. This completes the proof. $\qed$

\section{Proof of Lemma \ref{lemma3}} \label{prooflemma3}
We first prove (\ref{rankbound}). Without losing generality, we consider  the feature set in (\ref{rm_features}). For any trajectory segment  $\boldsymbol{\xi}_{t: {t+l}}\subseteq\boldsymbol{\xi}$,  from  (\ref{recoveryequationbyH}), we have known that there exists a costate $\boldsymbol\lambda_{ {t+l+1}}^*$  such that 
\begin{equation}\label{pftheo1_Hsolution1}
\boldsymbol{H}(t,l)
\begin{bmatrix}
\bar{\boldsymbol{\omega}}\\
\boldsymbol \lambda_{ {t+l+1}}^*
\end{bmatrix}=\boldsymbol{0}
\end{equation} holds,  where $\bar{\boldsymbol\omega}\neq \boldsymbol{0}$ is defined in (\ref{rm_weightvector}). Thus, the nullity of $\boldsymbol{H}(t,l)$ is at least one, which means
\begin{equation*}
\rank \boldsymbol{H}(t,l) \leq r+n-1.
\end{equation*}

We then prove (\ref{rankbound2}). When another relevant feature subset $\mathcal{\widetilde F}$   exists in $\mathcal{F}$ with associated weight vector $\boldsymbol{\widetilde{\omega}}$, we can similarly construct a weight vector $\breve{\boldsymbol\omega}$ corresponding to $\col\mathcal{F}$ as in 
(\ref{rm_weightvector}); that is, the weights in $\breve{\boldsymbol\omega}$ that correspond to  $\mathcal{\widetilde F}$ are from $\boldsymbol{\widetilde{\omega}}$ and otherwise zeros. Then  following the similar derivations as from (\ref{costfun_ex}) to (\ref{recoveryequationbyH}), we can obtain that there exists  $\breve{\boldsymbol\lambda}_{ {t+l+1}}\in \mathbb{R}^n$ such that 
\begin{equation}\label{pftheo1_Hsolution2}
\boldsymbol{H}(t,l)\begin{bmatrix}
\breve{\boldsymbol\omega}\\
\breve{\boldsymbol\lambda}_{ {t+l+1}}
\end{bmatrix}=\boldsymbol{0}.
\end{equation}
Since $\mathcal{\widetilde F}\neq\mathcal{F^*}$ or  $\boldsymbol{\widetilde \omega}\neq c_1\boldsymbol{\omega}$ implies $\breve{\boldsymbol\omega}\neq  c_2 \bar{\boldsymbol{\omega}}$ ($c_1$ and $c_2$ are some nonzero scalars),  based on (\ref{pftheo1_Hsolution1}) and  (\ref{pftheo1_Hsolution2}), it follows that the nullity of $\boldsymbol{H}(t,l)$ is at least two, i.e.,
\begin{equation*}
\rank \boldsymbol{H}(t,l) \leq  r+n-2.
\end{equation*}
This completes the proof. $\qed$
\end{document}